%% file: main.tex
\documentclass[letterpaper]{article} 
\usepackage{aaai24}  
\usepackage{times}  
\usepackage{helvet}  
\usepackage{courier}  
\usepackage[hyphens]{url}  
\usepackage{natbib}  
\usepackage{caption} 
\usepackage{caption}
\usepackage{amsmath,amssymb,amsfonts,amsthm}
\usepackage{xcolor}   
\usepackage{thm-restate}
\usepackage{ifdraft}
\usepackage[final]{graphicx}
\usepackage{subcaption}  
\usepackage{array}  
\usepackage{booktabs}  
\usepackage[ruled,linesnumbered,vlined]{algorithm2e}

\graphicspath{{figures/}}

\usepackage{mymacros}

\renewcommand{\vec}{\mathbf}

\newcommand{\sign}{\operatorname{sign}}

\newcommand{\vreset}{V_\mathrm{reset}}
\newcommand{\vtoggle}{V_\mathrm{toggle}}
\newcommand{\vset}{V_\mathrm{set}}
\newcommand{\vread}{V_\mathrm{read}}
\newcommand{\xlow}{X_\mathrm{low}}
\newcommand{\xhigh}{X_\mathrm{high}}

\newtheorem{example}{Example}
\newtheorem{theorem}{Theorem}
\newtheorem{lemma}{Lemma}
\newtheorem{definition}{Definition}

\newtheorem{proposition}{Proposition}

\title{On The Expressivity of Recurrent Neural Cascades}

\author {
    Nadezda Alexandrovna Knorozova\textsuperscript{\rm 1},
    Alessandro Ronca\textsuperscript{\rm 2}
}
\affiliations {
    \textsuperscript{\rm 1}RelationalAI\\
    \textsuperscript{\rm 2}University of Oxford\\
    nadezda.knorozova@relational.ai, alessandro.ronca@cs.ox.ac.uk
}

\begin{document}

\maketitle

\input{sections/abstract}

\input{sections/introduction}


\input{sections/dynamical-systems}

\input{sections/recurrent_neural_cascades_and_networks}

\input{sections/automata}

\input{sections/homomorphism}

\input{sections/symbol-grounding}

\input{sections/abstract-neurons}

\input{sections/implementation-flipflop-neurons}

\input{sections/expressivity-results}

\input{sections/group-neurons}

\input{sections/related-work}

\input{sections/conclusion}

\bibliography{bibliography}

\onecolumn
\clearpage




\setcounter{page}{1}

\input{sections/definition-groupfree}

\input{sections/proofs}

\input{sections/example-details}


\end{document}

%% file: sections/abstract.tex
\begin{abstract}
  Recurrent Neural Cascades (RNCs) are the recurrent neural networks with no
  cyclic dependencies among recurrent neurons. This class of recurrent 
  networks has received a lot
  of attention in practice. 
  Besides training methods for a fixed architecture such as backpropagation,
  the cascade architecture naturally allows for constructive learning methods,
  where recurrent nodes are added incrementally one at a time, often yielding
  smaller networks. 
  Furthermore, acyclicity amounts to a structural prior 
  that even for the same
  number of neurons yields a more favourable sample complexity compared to
  a fully-connected architecture.

  A central question is whether the advantages of the cascade architecture come
  at the cost of a reduced expressivity.
  We provide new insights into this question.
  We show that the regular languages captured by
  RNCs with sign and tanh activation with positive recurrent weights are 
  the \emph{star-free} regular languages. In order to establish our results
  we developed a novel framework where capabilities of RNCs are accessed by
  analysing which semigroups and groups a single neuron is able to implement.
  A notable implication of our framework is that RNCs can achieve the 
  expressivity of all regular languages by introducing neurons that can 
  implement groups.
\end{abstract}

%% file: sections/introduction.tex
\section{Introduction}
Recurrent Neural Cascades (RNCs) are a class of recurrent networks
that has been successfully applied in many different areas, including
information diffusion in social networks \cite{wang2017topological},
geological hazard predictions \cite{zhu2020landslide},
automated image annotation \cite{shin2016learning},
brain-computer inference \cite{zhang2018cascade},
and optics \cite{xu2020cascade}.
In the cascade architecture neurons can be layed out into a sequence so that 
every neuron has
access to the output of all preceding neurons as well as to the external input;
and at the same time, it has no dependency on the subsequent neurons.
Compared to fully-connected networks,
the cascade architecture has half of the connections. It immediately implies 
that RNCs have a more favourable \emph{sample complexity}, or dually better
generalisation capabilities. This is a consequence of the fact that
the VC dimension of recurrent networks depends directly on the number of
connections \cite{koiran1998vc}.
%

The acyclic structure of the cascade architecture naturally allows for so-called
\emph{constructive learning} methods,~cf.~\cite{fahlman1990recurrent,russell1999neural}. These methods construct the
network architecture dynamically during the training, often yielding 
smaller networks, faster training and improved generalisation.
One such method is recurrent cascade correlation, which builds the 
architecture incrementally adding one recurrent neuron at a time
\cite{fahlman1990recurrent}. 
Cascades emerge naturally here from the fact that existing nodes will not 
depend on nodes added later. RNCs also admit 
learning methods for fixed architectures, such as 
backpropagation through time, cf.\ \cite{WerbosBPTT},
where only the weights are learned. For these methods the advantage of the 
cascade architecture 
comes from the reduced number of weights.

A central question is whether the advantages of the cascade architecture come at
the cost of a reduced \emph{expressivity} compared to the fully-connected
architecture. 
The studies so far have shown that there exist regular languages that are not
captured by RNCs with monotone activation such as tanh 
\cite{giles1995constructive}. However, an exact characterisation of their
expressitivity is still missing. Furthermore, it is unclear whether the 
inability to capture all regular languages is a limitation of the cascade
architecture or rather of the considered activation functions. 
We continue this investigation and provide new insights in the capabilities of 
RNCs to capture regular languages.


\paragraph{Our contribution.}
We develop an analysis of the capabilities of RNCs establishing the following
expressivity results.
\begin{itemize}
  \item
    RNCs with sign or tanh activations capture 
    the star-free regular languages. 
    The expressivity result already holds when recurrent weights are 
    restricted to be positive. 
  \item
    RNCs with sign or tanh activations and positive recurrent weights do 
    not capture any regular language that is not star-free.
  \item
    Allowing for negative recurrent weights properly extends the expressivity
    of RNCs with sign and tanh activations beyond the star-free regular
    languages.
  \item
    We show that in principle the expressivity of RNCs can be extended to all 
    regular languages. It suffices to identify appropriate recurrent neurons.
    In particular, neurons that can implement finite simple groups. 
    As a first step, we show that second-order sign and tanh neurons can
    implement the cyclic group of order two.
\end{itemize}

Our expressivity results establish an important connection between 
recurrent neural 
networks and the wide range of formalisms whose expressivity is 
star-free regular languages. Such formalisms include  
\emph{star-free regular expressions}, cf.\ \citep{ginzburg},
\emph{Monadic First-order Logic} on finite linearly-ordered
domains, cf.\ \citep{mcnaughton1971counter}, 
\emph{Past Temporal Logic}, cf.\ \citep{manna1991completing}, and
\emph{Linear Temporal Logic} on finite traces 
\citep{degiacomo2013ltlf}.
They are also the languages recognised by 
\emph{counter-free automata acceptors}
as well as \emph{group-free automata acceptors} from where
they take their name, cf.\ 
\citep{schutzenberger1965finite,ginzburg,mcnaughton1971counter}.
On one hand, our results introduce an opportunity of employing
RNCs for learning targets that one would 
describe in any of the above formalisms.  
For such targets, RNCs are sufficiently expressive and,
compared to fully-connected recurrent neural networks, offer a more
favorable sample complexity along with a wider range 
of learning algorithms.
On the other hand, it places RNCs alongside well-understood formalisms with the
possibility of establishing further connections and leveraging many existing
fundamental results. 

As a result of our investigation we develop a novel framework where recurrent
neural networks are analysed through the lens of Semigroup and Group Theory. 
We 
establish a formal correspondence between continuous 
systems such as recurrent neural networks and discrete abstract objects such 
as semigroups and groups. Effectively we bridge RNCs with Algebraic
Automata Theory, two fields that developed independently, and so far have not 
been considered to have any interaction. 
Specifically, our framework allows for establishing the 
expressivity of RNCs by analysing the capabilities of a single neuron 
from the point of view of which semigroups and groups it can
implement.
If a neuron can implement the so-called \emph{flip-flop monoid},
then cascades of such neurons capture the star-free regular languages.
To go beyond that, 
it is sufficient to introduce neurons that implement \emph{groups}.
Our framework can be readily used to analyse the expressivity of RNCs with 
neurons that have not been considered in this work.
In particular, we introduce abstract flip-flop and group neurons, which are
the neural counterpart of the flip-flop monoid and of any given group.
To show expressivity results, it is sufficient to instantiate
our abstract neurons. Specifically in this work we show how to instantiate
flip-flop neurons with (first-order) sign and tanh, as well as a family of
grouplike neurons with second-order sign and tanh.
In a similar way, 
other results can be obtained by instantiating
the abstract neurons with different activation functions.

We provide proofs of all our results in the appendix as well as more extensive 
background on the required notions from the semigroup and  group theory. 
Furthermore, we also provide examples of star-free regular languages as found
in two different applications.

%% file: sections/dynamical-systems.tex
\section{Part I: Background}

We introduce the necessary background.

\subsection{Dynamical Systems}

Dynamical systems provide us with a formalism where to cast both neural networks
and automata. The kind of dynamical system relevant to us is
described next. It is discrete-time, and it has some continuity properties.
Specifically, a \emph{dynamical system} $S$ is a tuple
\begin{align*}
    S = \langle U, X, f, x^\mathrm{init}, Y, h \rangle,
\end{align*}
where $U$ is a set of elements called \emph{inputs},
$X$ is a set of elements called \emph{states},
$f: X \times U \to X$ is called \emph{dynamics function},
$x^\mathrm{init} \in X$ is called \emph{initial state},
$Y$ is a set of elements called \emph{outputs}, and
$h: X \times U \to Y$ is called \emph{output function};
furthermore, sets $U,X,Y$ are metric spaces, and functions $f,h$ are continuous.
At every time point $t = 1,2, \dots$, the system receives an input 
$u_t \in U$.
The state $x_t$ of the system at time $t$ is defined as follows. At time
$t=0$, before receiving any input, the system is in state $x_0 =
x^\mathrm{init}$. 
Then, the state $x_t$ and output $y_t$ are determined by the previous state 
$x_{t-1}$ and the current input $u_t$ as
\begin{align*}
    x_t = f(x_{t-1}, u_t), 
    \qquad 
    y_t = h(x_{t-1}, u_t).
\end{align*}
%
The \emph{dynamics} of $S$ are the tuple 
$D = \langle U, X, f \rangle$.
\emph{Subdynamics} of $D$ are any tuple $\langle U', X', f \rangle$ such that
$U' \subseteq U$, $X' \subseteq X$, and $f(X',U') \subseteq X'$. Note that
$f(X',U') = \{ f(x,u) \mid x \in X' \text{, } u \in U' \}$.
The function \emph{implemented} by system $S$ is the function that maps
every input sequence $u_1, \dots, u_\ell$ to the output sequence 
$y_1, \dots, y_\ell$.
We write $S(u_1, \dots, u_\ell) = y_1, \dots, y_\ell$.
Two systems are \emph{equivalent} if they implement the same function.



%

\paragraph{Architectures.}
A \emph{network} is a dynamical system $N$ with a factored state space  
$X  = X_1 \times \dots \times X_d$ and dynamics function of the
form 
\begin{align*}
  & f(\vec{x}, u) = \langle f_1(\vec{x}, u), \dots, f_d(\vec{x}, u) \rangle,
  \\
  & \text{where } \vec{x} = \langle x_1, \dots, x_d \rangle.
\end{align*}
Note that $f_i$ determines the $i$-th component of the state by reading the
entire state vector and the input.
A network can be expressed in a modular way by  
expressing the dynamics function as 
\begin{align*}
  & f(\vec{x}, u) = \langle f_1(x_1, u_1), 
  \dots, f_d(x_d, u_d) \rangle,
  \\
  & \text{where }
  u_i = \langle u, x_1, \dots, x_{i-1}, x_{i+1}, \dots, x_d \rangle.
\end{align*}
It is a \emph{modular view} because now the dynamics of $N$ can
be seen as made of the $d$ dynamics 
\begin{align*}
  D_i = \big\langle (U \times X_1 \times \dots \times X_{i-1} \times X_{i+1}
  \times \dots \times X_d), X_i, f_i \big\rangle.
\end{align*}
We call every $D_i$ a \emph{component} of the dynamics of $N$.
When there are no cyclic dependencies among the components of $N$,
the network can be expressed as a \emph{cascade}, where the dynamics
function is of the form 
\begin{align*}
  & f(\vec{x}, u) = \langle f_1(x_1, u_1), \dots, f_d(x_d, u_d) \rangle,
  \\
  & \text{where } u_i = \langle u, x_1, \dots, x_{i-1} \rangle.
\end{align*}
In a cascade, every component has access to the state of the preceding
components, in addition to the external input.
Differently, in a network, every component has access to the state of all 
components, in addition to the external input.

%% file: sections/recurrent_neural_cascades_and_networks.tex
\subsection{Recurrent Neural Cascades and Networks}

  A \emph{core recurrent neuron} is a triple 
  $N = \langle V , X, f \rangle$ where
  $V \subseteq \mathbb{R}$ is the input domain,
  $X \subseteq \mathbb{R}$ are the states, and 
  function $f$ is of the form
  \begin{align*}
    f(x, u) = \alpha((w \cdot x) \oplus v),
  \end{align*}
  with $w \in \mathbb{R}$ called \emph{weight}, 
  $\oplus$ a binary operator over $\mathbb{R}$,
  and
  $\alpha: \mathbb{R} \to \mathbb{R}$ called \emph{activation function}.
  A \emph{recurrent neuron} is the composition of a core recurrent neuron $N$
  with an \emph{input function} $\beta: U \subseteq \mathbb{R}^a \to V$ that can
  be implemented by a feedforward neural network.  
  Namely,
  it is a triple 
  $\langle U , X, f_\beta \rangle$ 
  where $f_\beta(x, u) = f(x, \beta(u))$.

We often omit the term `recurrent' as it is the only kind of
neuron we consider explicitly.
  By default we will assume that the operator $\oplus$ is addition.
  We will also consider the case where $\oplus$ is product; in this case we 
  refer to the neuron as a \emph{second-order neuron}.

  A neuron is a form of dynamics, so the notions introduced for dynamical
  systems apply.
  A \emph{Recurrent Neural Cascade (RNC)} is a cascade
  whose components are recurrent neurons and whose output function is a
  feedforward neural network. 
  A \emph{Recurrent Neural Network (RNN)} is a network
  whose components are recurrent neurons and whose output function is a
  feedforward neural network.

%% file: sections/automata.tex
\subsection{Automata}

Automata are dynamical systems with a finite input domain, a finite set of
states, and a finite output domain. The terminology used for automata is
different from the one used for dynamical systems.

The input and output domains are called \emph{alphabets}, and their
elements are called \emph{letters}.
Input and output sequences are seen as \emph{strings}, where a string
$\sigma_1 \dots \sigma_\ell$ is simply a concatenation of letters. The set of
all strings over an alphabet $\Sigma$ is written as $\Sigma^*$.

An automaton is a tuple
$A = \langle \Sigma, Q, \delta, q^\mathrm{init}, \Gamma, \theta \rangle$ 
where $\Sigma$ is called \emph{input alphabet} (rather than input domain), 
$Q$ is the set of states,
$\delta: Q \times \Sigma \to Q$ is called \emph{transition function} (rather
than dynamics function),
$q^\mathrm{init} \in Q$ is the initial state,
$\Gamma$ is called \emph{output alphabet} (rather than output domain), 
and
$\theta: Q \times \Sigma \to \Gamma$ is the output function.
Again, the requirement is that $\Sigma,Q,\Gamma$ are finite.
The tuple $D = \langle \Sigma, Q, \delta \rangle$ is called a 
\emph{semiautomaton}, rather than dynamics. 


In order to analyse automata, it is convenient to introduce the notion of state
transformation.
A \emph{state transformation} is a function $\tau: Q \to Q$ from states to
states, and it is called:
(i) a \emph{permutation} if\/ $\tau(Q) = Q$, 
(ii) a \emph{reset} if\/ 
$\tau(Q) = \{ q \}$ for some $q \in Q$,
(iii)
an \emph{identity} if $\tau(q)  = q$ for every $q \in Q$.
Note that an identity transformation is, in particular, a permutation
transformation.
Every input letter $\sigma \in \Sigma$ induces the state 
transformation $\delta_\sigma(q) = \delta(q,\sigma)$.
Such state transformation $\delta_\sigma$ describes all state updates triggered
by the input letter $\sigma$.
The \emph{set of state transformations} of semiautomaton $D$ is 
$\{ \delta_\sigma \mid \sigma \in \Sigma \}$.
Any two letters that induce the same state transformations are 
equivalent for the semiautomaton, in the sense that they trigger the same state
updates.
Such equivalence can be made explicit by writing a semiautomaton as consisting
of two components. The first component is an \emph{input function} that
translates input letters into letters of an \emph{internal alphabet} $\Pi$,
where each letter represents an equivalence class of inputs. The second
component is a semiautomaton operating on the internal alphabet $\Pi$.
This way, the internal letters induce distinct state transformations.

\begin{definition}
  Given a function $\phi: \Sigma \to \Pi$,
  and a semiautomaton $\langle \Pi, Q, \delta \rangle$
  their \emph{composition} is the semiautomaton
  $\langle \Sigma, Q, \delta_\phi \rangle$
  where $\delta_\phi$ is defined as 
  $\delta_\phi(q,\sigma) = \delta(q,\phi(\sigma))$.
  We call $\phi$ the \emph{input function} of the resulting semiautomaton, and
  we call $\Pi$ its \emph{internal alphabet}.
\end{definition}
We will often write semiautomata with an explicit input function---w.l.o.g.\
since we can always choose identity.

\subsection{Fundamentals of Algebraic Automata Theory}
\label{sec:algebraic_automata_theory}

Semiautomata can be represented as networks or cascades. This is a structured
alternative to state diagrams---the unstructured representation of its
transitions as a labelled graph.
For the cascade architecture, 
the fundamental theorem by \cite{krohn1965rhodes} shows that
every semiautomaton can be expressed as a cascade of so-called \emph{prime
semiautomata}.
Moreover, using only \emph{some} prime semiautomata, we obtain specialised
expressivity results.
Prime semiautomata can be partitioned into two classes. 
The first class of prime semiautomata are \emph{flip-flops}.
At their core, they have a semiautomaton that corresponds to the standard
notion of flip-flop from digital circuits, and hence they provide the
fundamental functionality of storing one bit of information, with the
possibility of setting and resetting.

\begin{definition}
  Let $\mathit{set}$, $\mathit{reset}$, $\mathit{read}$, and 
  $\mathit{high}$, $\mathit{low}$ be distinguished symbols.
  The core flip-flop semiautomaton is the semiautomaton $\langle \Pi, Q, \delta
  \rangle$ where
  the input alphabet is
  $\Pi = \{ \mathit{set}, \mathit{reset}, \mathit{read} \}$,
  the states are
  $Q = \{ \mathit{high}, \mathit{low} \}$,
  and the identities below hold:
  \begin{align*}
  \delta(\mathit{read},q) = q,
  \quad
  \delta(\mathit{set},q) = \mathit{high},
  \quad
  \delta(\mathit{reset},q) = \mathit{low}.
  \end{align*}
  A flip-flop semiautomaton is the composition of an input function with the
  core flip-flop semiautomaton.
\end{definition}
Note that the state transformations of a flip-flop characterise its intuitive
functionalities.
In particular,
$\mathit{read}$ induces an identity transformation,
$\mathit{set}$ induces a reset to $\mathit{high}$, and
$\mathit{reset}$ induces a reset to $\mathit{low}$.

%

The second class of prime semiautomata are the \emph{simple grouplike
semiautomata}.

\begin{definition}
  Let $G = (D, \circ)$ be a finite group.
  The core $G$ semiautomaton is the semiautomaton $\langle D, D, \delta
  \rangle$ where $\delta(g,h) = g \circ h$.
  A $G$ semiautomaton is the composition of an input function with the
  core $G$ semiautomaton.
  A semiautomaton is \emph{(simple) grouplike} if it is a $G$ semiautomaton for
  some (simple) group $G$.
\end{definition}

Of particular interest to us is the class of \emph{group-free} semiautomata. 
Intuitively, they are the semiautomata that do not involve groups, and do not
show any periodic behaviour. The formal definition requires additional notions
from semigroup theory, and hence we defer it to the appendix.
Next we state a direct implication of the \emph{Krohn-Rhodes decomposition
theorem}---see Theorem~3.1 of \citep{domosi2005algebraic}.
The statement of the theorem requires the notion of 
\emph{homomorphic representation} which is given later in
Definition~\ref{def:homomorphism} in a more general form that applies to
arbitrary dynamical systems.

\begin{theorem}[Krohn-Rhodes]
  \label{theorem:krohn_rhodes}
  Every semiautomaton is homomorphically represented by a cascade of prime
  semiautomata.
  Every group-free semiautomaton is homomorphically represented by a cascade of
  flip-flop semiautomata.
\end{theorem}
The converse of both statements in Theorem~\ref{theorem:krohn_rhodes} holds as
well. Thus, group-free semiautomata can be characterised as the semiautomata
that are homomorphically represented by a cascade of flip-flop semiautomata.
If one allows for cyclic dependencies, then
flip-flop semiautomata suffice to capture all semiautomata. This a direct
implication of the \emph{Letichevsky decomposition theorem}---see
Theorem~2.69 of \citep{domosi2005algebraic}.

\begin{theorem}[Letichevsky]
  \label{theorem:letichevsky}
  Every semiautomaton is homomorphically represented by a network of flip-flop
  semiautomata.
\end{theorem}

\subsection{Classes of Languages and Functions}

A \emph{language} $L$ over $\Sigma$ is a subset of $\Sigma^*$.
It can also be seen the indicator function $f_L: \Sigma^* \to \{ 0,1 \}$ where
$f_L(x) = 1$ iff $x \in L$.
In general we will be interested in functions 
$f: \Sigma^* \to \Gamma$ for $\Gamma$ an arbitrary output alphabet.

An \emph{acceptor} is a dynamical system whose output domain is $\{ 0,1 \}$.
The language \emph{recognised} by an acceptor is the set of strings on which the
acceptor returns $1$ as its last output.

We briefly summarise the classes of languages and functions that are relevant to
us.
The \emph{regular languages} are the ones recognised by automaton acceptors
\citep{kleene1956representation}.
The \emph{group-free regular languages} are the ones recognised by automaton
acceptors with a group-free semiautomaton, and they coincide with the
\emph{star-free regular languages}, cf.\ \citep{ginzburg}.
These notions can be naturally generalised to functions.
The \emph{regular functions} are the ones implemented by automata.
The \emph{group-free regular functions} are the ones implemented by automata
with a group-free semiautomaton.


%% file: sections/homomorphism.tex
\section{Part II: Our Framework}

In this section we present our framework for analysing RNCs. 
First, we introduce a notion of homomorphism for dynamical systems. Then, we
formalise the notion of symbol grounding. Finally, we introduce abstract neurons
which are the neural counterpart of prime semiautomata.

\subsection{Homomorphisms for Dynamical Systems}
\label{sec:homomorphism}

We introduce a new notion of homomorphism which allows us to compare
systems by comparing their dynamics.
Homomorphisms are a standard notion in automata theory, cf.
\citep{arbib1969theories}.
However, there, they do not deal with the notion of \emph{continuity}, which
holds trivially for all functions involved in automata, since they are over
finite domains.
Here we introduce homomorphisms for dynamical systems, with the requirement that
they must be continuous functions.
This allows one to infer results for continuous dynamical systems such as 
recurrent neural networks,
as stated by our Propositions~\ref{prop:homomorphism}
and~\ref{prop:homomorphism_converse}, which are instrumental to our
results.

%
%
%
%
%
%
%
%

%
\begin{definition}
  \label{def:homomorphism}
  Consider two system dynamics $D_1 = \langle U, X_1, f_1 \rangle$ and
  $D_2 = \langle U, X_2, f_2 \rangle$.
  A \emph{homomorphism} from $D_1$ to $D_2$ is a continuous surjective function
  $\psi: X_1 \to X_2$ satisfying the equality
  \begin{align*}
    \psi\big(f_1(x,u)\big) = f_2\big(\psi(x),u\big)
  \end{align*}
  for every state $x \in X_1$ and every input $u \in U$.
  We say that $D_1$ \emph{homomorphically represents} $D_2$ if $D_1$ has
  subdynamics $D_1'$ such that there is a homomorphism from $D_1'$ to $D_2$.
\end{definition}

\begin{restatable}{proposition}{prophomomorphism}
  \label{prop:homomorphism}
  If dynamics $D_1$ homomorphically represent dynamics $D_2$, then every system
  with dynamics $D_2$ admits an equivalent system with dynamics $D_1$.
\end{restatable}

For the second proposition, the following notions are needed, which are borrowed
from automata theory, cf. \citep{arbib1969theories}, but apply to dynamical
systems as well.

\begin{definition}
  A state $x$ of a system $S$ is \emph{reachable} if there is an input sequence
  $u_1, \dots, u_\ell$ such that the system is in state $x$ at time $\ell$.
  A system is \emph{connected} if every state is reachable.
  Given a system $S$ and one of its states $x$, the system $S^x$ is the system
  obtained by setting $x$ to be the initial state.
  Two states $x$ and $x'$ of $S$ are equivalent if the systems
  $S^x$ and $S^{x'}$ are equivalent.
  A system is in \emph{reduced form} if it has no distinct states which are
  equivalent. 
  A system is \emph{canonical} if it is connected and in reduced form.
\end{definition}

\begin{restatable}{proposition}{prophomomorphismconverse}
  \label{prop:homomorphism_converse}
  If a system $S_1$ is equivalent to a canonical system $S_2$ with a finite or
  discrete output domain, then 
  the dynamics of $S_1$ homomorphically represent the dynamics of~$S_2$.
\end{restatable}

%% file: sections/symbol-grounding.tex
\subsection{Symbol Grounding}

Our goal is to establish expressivity results for recurrent networks.
Given an input alphabet $\Sigma$ and an output alphabet $\Gamma$, we
want to establish which functions from $\Sigma^*$ to $\Gamma^*$ can be
implemented by a recurrent network.
However, recurrent networks operate on a real-valued input domain 
$U \subseteq \mathbb{R}^n$ and a real-valued output domain 
$Y \subseteq \mathbb{R}^m$.
In order to close the gap,
we introduce the notion of symbol grounding.

\begin{definition}
  Given a domain $Z \subseteq \mathbb{R}^n$ and an alphabet $\Lambda$,
  a \emph{symbol grounding} from $Z$ to $\Lambda$ is a continuous surjective
  function $\lambda: Z \to \Lambda$.
\end{definition}

Symbol groundings can be seen as connecting the subymbolic level $Z \subseteq
\mathbb{R}^n$ to the symbolic level $\Lambda$.
For an element $z$ at the subsymbolic level, the letter $\lambda(z)$ is its
meaning at the symbolic level.
Assuming that a symbol grounding $\lambda$ is surjective means that every
letter corresponds to at least one element $z \in Z$. 
The assumption is w.l.o.g.\ because we can remove the letters that do not
represent any element of the subsymbolic level.

Symbol groundings can be robust to noise, when every letter corresponds to a
ball in $\mathbb{R}^n$ rather than to a single point, so as to allow some
tolerance---any noise that stays within the ball does not affect the
symbolic level.
\begin{definition}
  A symbol grounding $\lambda: Z \to \Lambda$ is \emph{robust} if,
  for every $a \in \Lambda$, there exists a ball $B \subseteq Z$ of non-zero
  radius such that $\lambda(B) = \{ a \}$.
\end{definition}

%
%

In the following sections we establish expressivity results 
considering fixed, but arbitrary, 
input alphabet $\Sigma$, 
input domain $U \subseteq \mathbb{R}^n$, 
input symbol grounding $\lambda_\Sigma: U \to \Sigma$,
output alphabet $\Gamma$, 
output domain $Y \subseteq \mathbb{R}^m$, and 
output symbol grounding $\lambda_\Gamma: Y \to \Gamma$.
Whenever we relate a recurrent network, a RNC or a RNN, to an automaton, 
we mean that the network operates at the subsymbolic level and then its output
is mapped to the symbolic level, while the automaton operates entirely at the
symbolic level.
Formally, given a recurrent network
$\langle U, X, f, x^\mathrm{init}, Y, g \rangle$
and an automaton 
$\langle \Sigma, Q, \delta, q^\mathrm{init}, \Gamma, \theta \rangle$,
we relate the corresponding dynamical sytems 
$\langle U, X,f, x^\mathrm{init}, \Gamma, g \circ \lambda_\Gamma \rangle$
and
$\langle U, Q, \delta_{\lambda_\Sigma}, q^\mathrm{init}, \Gamma, \theta \rangle$
where $\delta_{\lambda_\Sigma}(q,u) = \delta(q,\lambda_\Sigma(u))$.
Note that both systems take inputs in $U$ and return letters in $\Gamma$.

\paragraph{Assumptions.}
We make the mild technical assumptions that 
$U$ is a \emph{compact}, 
and that the output symbol grounding $\lambda_\Gamma$ is \emph{robust}.
These assumptions, together with continuity, allow us to make use of the
Universal Approximation Theorem for feedforward neural networks, cf.\
\citep{hornik1991approximation}.


%% file: sections/abstract-neurons.tex
\subsection{Abstract Neurons}
\label{sec:rncs_of_flipflop_neurons}

We first introduce an abstract class of neurons that model the
behaviour of a flip-flop or grouplike semiautomaton. This allows us to state
general results about cascades and networks of such abstract neurons. Then, we 
show that this results will transfer to cascades and networks of any concrete
instantiation of such neurons.

\begin{definition}
  A \emph{core flip-flop neuron} is a core neuron $\langle V, X, f \rangle$ 
  where the set $V$ of inputs is expressed as the union of three disjoint 
  closed
  intervals $\vset,\vreset,\vread$ of non-zero length, the set $X$ of states is
  expressed as the union of two disjoint closed intervals $\xlow,\xhigh$, and
  the following conditions hold:
  \begin{align*}
    f(X,\vset) & \subseteq \xhigh,
    \\
    f(X,\vreset) & \subseteq \xlow,
    \\
    f(\xhigh,\vread) & \subseteq \xhigh,
    \\
    f(\xlow,\vread) & \subseteq \xlow.
  \end{align*}
  A \emph{flip-flop neuron} is the composition of a core flip-flop neuron with
  an input function.
  The \emph{state interpretation} of a flip-flop neuron
  is the function $\psi$ defined as 
  $\psi(x) = \mathit{high}$ for $x \in \xhigh$ and
  $\psi(x) = \mathit{low}$ for $x \in \xlow$.
\end{definition}

\begin{definition}
 \label{def:abstract_group_neuron}
  Let $G = (D, \circ)$ be a group with $D = \{ 1, \dots, n \}$.
  A \emph{core $G$ neuron} is a core neuron $\langle V, X, f \rangle$ 
  where the set $V$ of inputs is expressed as the union of $n$ disjoint 
  closed intervals $V_1, \dots, V_n$ of non-zero length, the set $X$ of states is
  expressed as the union of $n$ disjoint closed intervals $X_1, \dots, X_n$, and
  the following condition holds for every $i,j \in D$:
  \begin{align*}
    f(X_i, V_j) \subseteq X_{i \circ j}.
  \end{align*}
  A \emph{$G$ neuron} is the composition of a core $G$ neuron with
  an input function.
  The \emph{state interpretation} of a $G$ neuron
  is the function $\psi$ defined as 
  $\psi(x) = i$ for $x \in X_i$.
  A neuron is \emph{(simple) grouplike} if it is a $G$ neuron for
  some (simple) group $G$.
\end{definition}

Abstract neurons are designed to be the neural counterpart of 
flip-flop and grouplike semiautomata. Specifically, they are designed to
guarantee the existence of a homomorphism as stated in
Lemma~\ref{lemma:flipflop_to_sign_neuron} below. The
lemma and all the following results involving abstract neurons hold regardless
of the specific way the core of an abstract neuron is instantiated. 
We highlight this aspect in the claims by referring to an abstract neuron
\emph{with arbitrary core}. 

\begin{restatable}{lemma}{lemmaflipfloptosignneuron}
  \label{lemma:flipflop_to_sign_neuron}
  Every flip-flop semiautomaton is homomorphically represented by
  a flip-flop neuron with arbitrary core.
  Similarly, every $G$ semiautomaton is homomorphically represented by
  a $G$ neuron with arbitrary core.
  In either case, 
  the homomorphism is given by the state interpretation of the neuron.
\end{restatable}

The lemma is based on three key observations. First, the inclusion requirements
in the definition of a flip-flop neuron determine a correspondence with
transitions of a flip-flop semiautomaton; the same holds for grouplike neurons.
Second, the fact that input intervals have non-zero length 
introduces sufficient tolerance to approximate the input function of a
semiautomaton by a feedforward neural network making use of
Universal Approximation Theorems, cf.\
\citep{hornik1991approximation}.
Third, the fact that state partitions are closed intervals ensures continuity of
a homomorphism.

The previous lemma extends to cascades and networks.


\begin{restatable}{lemma}{lemmaautomatoncascadetoneuralcascade}
  \label{lemma:automaton_cascade_to_neural_cascade}
  Every cascade (or network) of flip-flop or grouplike semiautomata
  $A_1, \dots, A_d$ is homomorphically represented
  by a cascade (network, resp.) of d neurons
  $N_1, \dots, N_d$ where $N_i$ is a flip-flop neuron if $A_i$ is a 
  flip-flop semiautomaton and it is a $G$ neuron if $A_i$ is a 
  $G$ semiautomaton.
\end{restatable}
%

%% file: sections/implementation-flipflop-neurons.tex
\section{Part III: Expressivity Results}

We present our results for RNCs of sign and tanh activation.

\subsection{Implementation of Flip-Flop Neurons}
\label{sec:expressivity_rncs_sign_tanh_nonnegative}
We give precise conditions under which neurons with sign and tanh activation 
are flip-flop neurons. Following the definition of the abstract flip-flop
neuron the goal is to partition the state space of both sign and tanh into
low and high states and then find inputs inducing read, set and reset 
transitions.
For sign activation the choice is simple, we 
interpret -1 as the low state and +1 as the high state.
Since states are bounded, we know the maximum and minimum value that can be
achieved by $w \cdot x$ for any possible state $x$.
Therefore we can find inputs that will either maintain the sign or make it the
desired one.

\begin{restatable}{proposition}{propcoreflipflopsign}
  \label{prop:sign_transitions}
  Let $w >0$. A core neuron with sign activation and weight $w$
  is a core flip-flop neuron if
  its state partition is 
  \begin{align*}
      \xlow = \{ -1 \}, \; \xhigh = \{ +1 \}, 
  \end{align*}
  for some real number $a \in (0, 1)$
  and its inputs partition satisfies
  \begin{align*}
  \vreset & \in \bigl(-\infty,\, w \cdot (-a -1)\bigr]
    \\
    \vread & \in \bigl[w \cdot (a - 1),\, w \cdot (1 - a)\bigr]
    \\
    \vset & \in \bigl[w \cdot (a + 1),\, +\infty\bigr).
  \end{align*}
\end{restatable}

The tanh activation requires a more careful treatment. We represent 
the low and the high states as closed disjoint intervals including -1 and 
+1 respectively. Then using the values of state boundaries and the 
monotonicity property of tanh we can find inputs allowing for read, set and reset 
transitions without violating the state boundaries.


\begin{restatable}{proposition}{propcoreflipfloptanh}
  \label{prop:tanh_transitions}
  Let $w > 1$, and let $f(x) = \tanh(w \cdot x)$.
  A core neuron with tanh activation and weight $w$ is a core flip-flop neuron
  if its state partition is
  \begin{align*}
    \xlow = [-1, f(a)],
    \qquad
    \xhigh = [f(b), +1],
  \end{align*}
  for some real numbers $a < b$ satisfying 
  $a - f(a) > b - f(b)$,
  and its input partition satisfies
  \begin{align*}
  \vreset & \in  \bigl(-\infty,\, w \cdot (a - 1) \bigr],
  \\
   \vread & \in \bigl[w \cdot (b - f(b)),\,  w \cdot (a - f(a)) \bigr],
   \\
   \vset &\in \bigl[w \cdot (b + 1),\, +\infty\bigr).
  \end{align*}
\end{restatable}
Differently from sign activation the low and high states of tanh are not 
partitioned based on their sign. In fact, the low states can include
positive values and high states can include negative values. This
is determined entirely by the values of $a$ and $b$ defining the state
boundaries. We remark that the range of valid $a,b$ values 
increases with the increasing value of $w$.  
The quantities $a,b$ also determine the length
of $\vread$ interval, that impacts the robustness or the noise tolerance of 
the neuron. It is possible to choose the values of $a$ and $b$
that maximise the length of the $\vread$ interval. In particular, these are 
the points where the derivative of $f(x)$ is equal to one. 



%% file: sections/expressivity-results.tex
\subsection{Expressivity of RNCs}
\label{sec:expressivity-results}

We are now ready to present our expressivity results. Note that we state them
for functions, but they apply to languages as well, since they correpond to
functions as discussed in the background section.

As a positive expressivity result, we show that RNCs capture group-free regular
functions. 

\begin{restatable}{theorem}{theoremexpressivityofrnc} 
  \label{theorem:expressivity_of_rnc}
  Every group-free regular function can be implemented by an 
  RNC of flip-flop neurons with arbitrary core.
  In particular, it can be implemented by an RNC of neurons with sign or tanh
  activation, where it is sufficient to consider positive weights.
\end{restatable}

The result is obtained by applying results from the previous sections.
We have that every group-free regular function $F$ is implemented by a
group-free automaton, whose semiautomaton is homomorphically represented by
a cascade of flip-flop semiautomata
(Theorem~\ref{theorem:krohn_rhodes}), which is in turn homomorphically
represented by a cascade of flip-flop neurons
(Lemma~\ref{lemma:automaton_cascade_to_neural_cascade}); therefore,
$F$ is implemented by a system whose dynamics are a cascade of flip-flop neurons
(Proposition~\ref{prop:homomorphism}) and whose output function is some
continuous output function; we replace the output function with a feedforward
neural network making use of the Universal Approximation Theorem, relying on the
fact that approximation will not affect the result because the output symbol
grounding is assumed to be robust.

In the rest of this section we show that RNCs of sign or tanh neurons with
positive weight do not implement regular functions that are not group-free.
In order to go beyond group-free regular functions, it is necessary for the
dynamics to show a periodic, alternating behaviour.
\begin{restatable}{lemma}{lemmapermutationalternatingstates}
  \label{lemma:permutation_alternating_states}
  If a semiautomaton that is not group-free 
  is homomorphically represented by dynamics $\langle U, X, f \rangle$, with
  homomorphism $\psi$, then there exist $u \in U$ and $x_0 \in X$ such
  that, for $x_i = f(x_{i-1},u)$, 
  the disequality $\psi(x_i) \neq \psi(x_{i+1})$
  holds for every $i \geq 0$.
\end{restatable}

Considering the possibility that a sign or tanh neuron with positive weight
can satisfy the condition above, we show that a constant input yields a
convergent sequence of states. In fact, more generally, such a sequence is
convergent even when the input is not constant but itself convergent---this
stronger property is required in the proof of the lemma below. 
Overall, we show that a cascade of sign or tanh neurons with positive weight
can capture a group-free semiautomaton only at the cost of
generating a sequence of converging alternating states. This would amount to an
essential discontinuity for any candidate homomorphism.

\begin{restatable}{lemma}{lemmagroupfree}
  \label{lemma:groupfree}
  Every semiautomaton that is not group-free is not
  homomorphically represented by a cascade where each component is a neuron
  with sign or tanh activation and positive weight.
\end{restatable}

Then the expressivity result follows from the lemma by
Proposition~\ref{prop:homomorphism_converse}.
\begin{restatable}{theorem}{theoremmainonetwo} 
  \label{theorem:main_1_2}
  For any regular function $F$ that is not group-free, there is no RNC
  implementing $F$ whose components are neurons with sign or tanh activation and
  positive weight.
\end{restatable}

In light of Theorem~\ref{theorem:expressivity_of_rnc} and
Theorem~\ref{theorem:main_1_2}, we identify a class of RNCs that can
implement all group-free regular functions and no other regular function.
\begin{restatable}{theorem}{expressivityofrncwithsignandtanhwithposweights}
  \label{theorem:expressivity_of_rnc_with_sign_tanh_positive_weight}
  The class of regular functions that can be implemented by RNCs of sign or tanh
  neurons with positive weight is the group-free regular functions.
\end{restatable}

\section{Necessary Conditions for Group-freeness}
\label{sec:necessary_conditions}

We show that both acyclicity and positive weights of sign and tanh are
necessary to stay within the group-free functions. 
First, recurrent neural networks, with arbitrary dependencies among their
neurons, implement all regular functions, including the ones that are not
group-free. 
\begin{restatable}{theorem}{theoremexpressivityofrnn} 
  \label{theorem:expressivity_of_rnn}
  Every regular function can be implemented by an 
  RNN of flip-flop neurons with arbitrary core.
  In particular, it can be implemented by an RNN of neurons with sign or tanh
  activation.
\end{restatable}
The theorem is proved similarly to Theorem~\ref{theorem:expressivity_of_rnc}, 
using Theorem~\ref{theorem:letichevsky} in place of
Theorem~\ref{theorem:krohn_rhodes}.
The above Theorem~\ref{theorem:expressivity_of_rnn} seems to be folklore.
However we are not aware of an existing formal proof for the case of
a differentiable activation function such as tanh. 
We discuss it further in the related work section. 

Next we show that the restriction to positive weights is necessary to be
group-free.
\begin{restatable}{theorem}{theoremexpressivitytoggle}
  \label{theorem:expressivity_toggle}
  There is an RNC consisting of a single sign or tanh neuron with
  negative weight that implements a regular function that is not group-free
  regular.
\end{restatable}
The proof amounts to showing that a sign or tanh neuron with negative weight
captures a semiautomaton that is not group-free. It is a two-state semiautomaton
with one non-identity permutation transformation. 
We conjecture that sign or tanh neurons are not able to capture an
actual grouplike semiautomaton.

%% file: sections/group-neurons.tex
\section{Implementation of Group Neurons}
We give an instantiation of a group neuron as defined in the
Definition~\ref{def:abstract_group_neuron}. In particular, we show
when second-order neurons with sign or tanh activations are instances of 
the $C_2$ neuron, a neuron implementing cyclic group of order two. 

\begin{restatable}{proposition}{propctwosign}
\label{prop:ctwo_sign}
Let $w, a$ be real numbers either satisfying $a, w > 0$ or 
$a, w < 0$. 
A core second-order neuron with $\sign$ activation and weight $w$
is a core $C_2$-neuron, if its states partition is
 \begin{align*}
 X_0 = \{-1\}, \; X_1 = \{+1\},
 \end{align*}
and its input partition satisfies
\begin{align*}
V_1 \in (-\infty, -a] &, V_0 \in [a, +\infty), &\text{ if } a,w > 0,
\\
V_0 \in (-\infty, a] &, V_1 \in [-a, +\infty), &\text{ if } a,w < 0.
\end{align*}
\end{restatable}

\begin{restatable}{proposition}{propctwotanh}
\label{prop:ctwo_tanh}
Let $w, a$ be real numbers either satisfying $a, w > 0$ or 
$a, w < 0$. 
Let $f(x) = \tanh(w \cdot x)$.  
A core second-order neuron with $tanh$ activation and weight $w$
is a core $C_2$-neuron, if its states partition is
 \begin{align*}
 X_0 = [-1, -f(a)], \;\;
 X_1 = [f(a), +1]
 \end{align*}
and its input partition satisfies
\begin{align*}
V_1 \in (-\infty, -a/f(a)]&, V_0 \in [a/f(a), +\infty), &\text{if } a,w > 0,
\\
V_0 \in (-\infty, a/f(a)]&, V_1 \in [-a/f(a), +\infty), &\text{if } a,w < 0.
\end{align*}
\end{restatable}

Then by Lemma~\ref{lemma:flipflop_to_sign_neuron} the above neurons 
homomorphically represent $C_2$ semiautomata.
By Lemma~\ref{lemma:automaton_cascade_to_neural_cascade} an 
RNC containing these neurons can homomorphically represent a cascade of 
$C_2$ semiautomata. In particular, such RNCs can recognise languages that
are not star-free, cf.\ \cite{ginzburg}.

%% file: sections/related-work.tex
\section{Related Work}

In our work, the connection between RNNs and automata plays an important role.
Interestingly, the connection appears to exist from the beginning of automata
theory \citep{arbib1969theories}: 
``\emph{In 1956 the series Automata Studies  (Shannonon and McCarthy [1956]) was
  published, and automata theory emerged as a relatively autononmous discipline.
  [...] much interest centered on finite-state sequential machines, which first
arose not in the abstract form [...], but in connection with the
input-output behaviour of a McCulloch-Pitts net [...]}''.
The relationship between automata and the networks by
\cite{mcculloch1943logical} is discussed both in
\citep{kleene1956representation}
and
\citep{minsky1967computation}.
Specifically, an arbitrary automaton can be captured by a McCulloch-Pitts
network.
Our Theorem~\ref{theorem:expressivity_of_rnn} reinforces this result,
extending it to sign and tanh activation. The extension to tanh is important
because of its differentiability, and it requires a different set of techniques
since it is not binary, but rather real-valued.
Furthermore, our results extend theirs by showing a correspondence between
RNCs and group-free automata.

The Turing-completeness capabilities of RNN as an \emph{offline model}
of computation 
are studied in
\citep{siegelmann1995turing,siegelmann1996dynamic,hobbs2015implementation,chung2021turing}.
In this setting, an RNN is allowed to first read the entire input sequence, 
and then return the output with an arbitrary delay---the so-called computation time.
This differs from our study, which focuses on the capabilities of RNNs as online
machines, which process the input sequence one element at a time, outputting
a value at every step. 
This is the way they are used in many practical applications such as
Reinforcement Learning, cf. 
\citep{bakker2001rlandlstm,ha2018worldmodels,stone2015rnn,kapturowski2019rnnreplay}.

The expressivity of RNNs in terms of whether
they capture all \emph{rational series} or not has been analysed in \cite{merrill2020formal}. This is a class of functions
that includes all regular functions. Thus, it is a coarse-grained analysis
compared to ours, which focuses on subclasses of the regular languages. 

The problem of \emph{latching} one bit of information has been studied in
\cite{bengio1994latching} and
\cite{frasconi1995unified}.
This problem is related to star-free regular languages, as it amounts to asking
whether there is an automaton recognising a language of the form $sr^*$ where
$s$ is a set command and $r$ is a read command. This is a subset of the
functionalities implemented by a flip-flop semiautomaton. Their work established
conditions under which a tanh neuron can latch a bit.
Here we establish conditions guaranteeing that a tanh neuron homomorphically
represents a flip-flop semiautomaton, implying that it can latch a bit.
An architecture that amounts to a restricted class of RNCs has been considered
in \cite{frasconi1992local}.

Automata cascades are considered in \citep{ronca2023cascades},
where they are shown to yield favourable sample complexity results for automata
learning. 


%% file: sections/conclusion.tex
\section{Conclusions and Future Work}
We developed a new methodology that provides a fresh perspective on RNCs as 
systems implementing semigroups and groups.
This enabled us to establish new expressivity results for RNCs with sign 
and tanh activations.
We believe our methodology has a potential that extends beyond our current
results. In particular, we believe it provides a principled way to identify new
classes of recurrent networks that incorporate different priors based on groups.

We have covered sign and tanh activation, postponing the study of other
activation functions such as logistic curve, ReLU, GeLU.
Beyond that, one could identify neurons that can homomorphically represent
grouplike semiautomata.  
This will allow to capture specific subclasses of regular functions that are
beyond group-free. To this extend we presented second-order sign and tanh 
neurons as 
instances of neurons homomorphically representing the cyclic group of order two. 


%% file: sections/definition-groupfree.tex
The appendix consists of three parts.
\begin{enumerate}
  \item
    We first provide some additional background in Section `Algebraic Notions
    for Automata'.
  \item
    We present the proofs of all our results, following the same order as they
    appear in the main body.
  \item
    We give examples of a star-free (group-free) regular language and of a
    a group-free regular function.
    In particular we provide a detailed construction of the cascades capturing
    them.
\end{enumerate}

\section{Algebraic Notions for Automata}

In this section we provide the formal definition of 
\emph{group-free semiautomaton}, which
we avoided to provide in the main sections, as it requires some preliminary 
notions. 
We also define the notion of \emph{characteristic semigroup} of an automaton.
We follow \citep{ginzburg,domosi2005algebraic}.
These notions require some algebraic preliminaries, which are also required by
some of the proofs.

\subsection{Algebraic Preliminaries}
A \emph{semigroup} is a non-empty set together with an \emph{associative} binary
operation that combines any two elements $a$ and $b$ of the set to form a third
element $c$ of the set, written $c = (a \cdot b)$.
A \emph{monoid} is a semigroup that has an \emph{identity element} $e$, i.e., 
$(a \cdot e) = (e \cdot a) = a$ for every element $a$. The identity element is
unique when it exists. 
A \emph{flip-flop monoid} is a three-element monoid $\{s,r,e\}$
where $(r \cdot s) = s$, $(s \cdot s) = s$,
$(r \cdot r) = r$, $(s \cdot r) = r$.
A \emph{group} is a monoid where every element $a$ has an \emph{inverse}
$b$, i.e., $(a \cdot b) = (b \cdot a) = e$ where $e$ is the identity element.
A \emph{subsemigroup} (\emph{subgroup}) of a semigroup $G$ is a subset of $G$
that is a semigroup (group).

Let $G$ be a group and let $H$ be a subgroup of $G$. 
For every $g \in G$, the \emph{right coset} of $G$ is 
$gH = \{  gh \mid h \in H\}$, and its \emph{left coset} is
$Hg = \{  hg \mid h \in H\}$. Subgroup $H$ is \emph{normal} if its left and
right cosets coincide, i.e., $gH = Hg$ for every $g \in G$.
A group is \emph{trivial} if it is $\{ e \}$ where $e$ is the identity element.
A \emph{simple group} is a group $G$ such that every normal subgroup of $G$ is
either trivial or $G$ itself.

A \emph{homomorphism} from a semigroup $S$ to a semigroup $T$ is a mapping 
$\psi:S \to T$ such that $\psi(s_1 \cdot s_2) = \psi(s_1) \cdot \psi(s_2)$ for
every $s_1,s_2 \in S$.
If $\psi$ is surjective, we say that $T$ is a \emph{homomorphic image} of $S$.
If $\psi$ is bijective, we say that $S$ and $T$ are \emph{isomorphic}, and
$\psi$ is called an \emph{isomorphism}.
A semigroup $S$ \emph{divides} a semigroup $T$ if $T$ has a subsemigroup $T'$
such that $S$ is a homomorphic image of $T'$.
If $G$ is \emph{simple}, then every homomorphic image of $G$ is isomorphic to
$\{ e \}$ or $G$.
For $G$ and $H$ semigroups, we write 
$GH = \{ g\cdot h \mid g \in G, h \in H\}$. We also write 
$H^1 = H$ and $H^k = HH^{k-1}$.
A semigroup $S$ is \emph{generated} by a semigroup $H$ if $S = \bigcup_n
H^n$; then $H$ is called a \emph{generator} of $S$.

\subsection{Definition of Group-free and Grouplike Semiautomata}

The \emph{characteristic semigroup} of a semiautomaton
$A$ is the semigroup generated by its transformations. 
A semiautomaton is \emph{group-free} if its characteristic semigroup has no
divisor which is a non-trivial group, cf.\ Page~153 of \citep{ginzburg}.

%% file: sections/proofs.tex
\section{Proofs for Section `Homomorphisms for Dynamical Systems'}

\subsection{Proof of Proposition~\ref{prop:homomorphism}}
We prove Proposition~\ref{prop:homomorphism}.

\prophomomorphism*
\begin{proof}
  Let us consider dynamics $D_1$ and $D_2$.
  Assume that $D_1$ homomorphically represent $D_2$.
  There exist subdynamics $D_1'$ of $D_1$ such that there is a homomorphism
  $\psi$ from $D_1'$ to $D_2$. 
  \begin{align*}
    D_1' & = \langle U, X, f_1 \rangle
    \\
    D_2 & = \langle U, Z, f_2 \rangle
  \end{align*}
  Consider a system $S_2$ with dynamics $D_2$.
  \begin{align*}
    S_2 = \langle U, Z, f_2, z^\mathrm{init}, Y, h_2 \rangle
  \end{align*}
  We construct the system 
  \begin{align*}
    S_1 = \langle U, X, f_1, x^\mathrm{init}, Y, h_1 \rangle,
  \end{align*}
  where $h_1(x,u) = h_2(\psi(x),u)$, and
  $x^\mathrm{init} \in X_1$ is such that $\psi(x^\mathrm{init}) =
  z^\mathrm{init}$---it exists because $\psi$ is surjective.
  Note that $h_1$ is continuous as required, since it is the composition of
  continuous functions.

  We show that $S_1$ is equivalent to as $S_2$ as required.
  Let $u_1, \dots, u_n$ be an input sequence,
  and let $x_0, \dots, x_n$ and and $z_0, \dots, z_n$ be the corresponding
  sequence of states for $S_1$ and $S_2$, respectively.
  Namely, $x_0 = x^\mathrm{init}$ and 
  $x_i = f_1(x_{i-1}, u_i)$ for $1 \leq i \leq n$.
  Similarly, $z_0 = z^\mathrm{init}$ and 
  $z_i = f_2(z_{i-1}, u_i)$ for $1 \leq i \leq n$.

  As an auxiliary result, we show that
  $z_i = \psi(x_i)$ for every $0 \leq i \leq n$.
  We show it by induction on $i$ from $0$ to $n$.

  In the base case $i=0$, and $z_0 = \psi(x_0)$  amounts to
  $z^\mathrm{init} = \psi(x^\mathrm{init})$, which holds by construction.

  In the inductive case $i > 0$ and we assume that 
  $z_{i-1} = \psi(x_{i-1})$.
  We have to show $z_i = \psi(x_i)$. By the definition of $x_i$ and $z_i$ above,
  it can be rewritten as 
  \begin{align*}
    f_2(z_{i-1},u_i) = \psi(f_1(x_{i-1},u_i))
    \\
    \psi(f_1(x_{i-1},u_i)) = f_2(z_{i-1},u_i)
  \end{align*}
  Then, by the inductive hypothesis
  $z_{i-1} = \psi(x_{i-1})$, we have
  \begin{align*}
    \psi(f_1(x_{i-1},u_i)) = f_2(\psi(x_{i-1}),u_i) 
  \end{align*}
  which holds since
  $\psi$ is a homomorphism from $D_1$ to $D_2$,
  This proves the auxiliary claim.

  Now, to show that $S_1$ and $S_2$ are equivalent, it suffices to show 
  $S_1(u_1, \dots, u_\ell) = S_2(u_1, \dots, u_\ell)$.
  By definition, we have that 
  $S_1(u_1, \dots, u_\ell) = y_1, \dots, y_n$ with 
  $y_i = h_1(x_{i-1},u_i)$, and similarly
  $S_2(u_1, \dots, u_\ell) = w_1, \dots, w_n$ with 
  $w_i = h_2(z_{i-1},u_i)$.
  Then, for every $1 \leq i \leq n$,
  we have that 
  \begin{align*}
    w_i = h_2(z_{i-1},u_i) = 
    h_2(\psi(x_{i-1}),u_i) = 
    h_1(x_{i-1},u_i) = y_i.
  \end{align*}
  Therefore $S_1(u_1, \dots, u_\ell) = S_2(u_1, \dots, u_\ell)$ as required.
  This proves the proposition.
\end{proof}

\subsection{Proof of Proposition~\ref{prop:homomorphism_converse}}
The proof of Proposition~\ref{prop:homomorphism_converse} requires some 
preliminary definitions and propositions. 
Below we state the required 
Definition~\ref{def:definition_for_prop_two_one} and 
Definition~\ref{def:definition_for_prop_two_two}.
We also state and prove the required 
Proposition~\ref{prop:statesim-implies-funcsim},
Proposition~\ref{prop:equivalence-classes-for-canonical-system},
and
Proposition~\ref{prop:equivalence-classes-for-connected-system}.
Then we prove Proposition~\ref{prop:homomorphism_converse}.

\begin{definition}
  \label{def:definition_for_prop_two_one}
  Consider a function $F: U^* \to Y^*$.
  Two strings $s,s' \in U^*$ are in relation $s \sim_F s'$ iff the equality 
  $F(sz) = F(s'z)$ holds for every non-empty string $z \in U^+$.
  The resulting set of equivalence classes is written as $U^*/{\sim_F}$, and the
  equivalence class of a string $s \in U^*$ is written as $[s]_F$.
\end{definition}

\begin{definition}
  \label{def:definition_for_prop_two_two}
  Consider a system $S$.
  Two input strings $s,s'$ are in relation $s \sim_S s'$ iff
  the state of $S$ upon reading $s$ and $s'$ is the same.
  The resulting set of equivalence classes is written as $U^*/{\sim_S}$, and the
  equivalence class of a string $s \in U^*$ is written as $[s]_S$.
\end{definition}

\begin{proposition} 
  \label{prop:statesim-implies-funcsim}
  For every system $S$ that implements a function $F$,
  it holds that $s \sim_S s'$ implies $s \sim_F s'$.
\end{proposition}
\begin{proof}
  Assume $s \sim_S s'$, i.e., the two strings lead to the same state $x$.
  For every non-empty string $z$, the output of $S$ on both $sz$ and $s'z$ is 
  $S_x(z)$. Since $S$ implements $F$, it follows that $F(sz) = S_x(z)$ and
  $F(s'z) = S_x(z)$. Thus $F(sz) = F(s'z)$, and hence $s \sim_F s'$.
\end{proof}

\begin{proposition} 
  \label{prop:equivalence-classes-for-canonical-system}
  Consider a canonical system $S$ on input domain $U$ that implements a function
  $F$. 
  The states of $S$ are in a one-to-one correspondence with the equivalence
  classes $U^*/{\sim_F}$. 
  State $x$ corresponds to the equivalence class $[s]_F$ for $s$ any string that
  leads to $x$.
\end{proposition}
\begin{proof}
  Let $\psi$ be the mentioned correspondence.
  First, $\psi$ maps every state to some equivalence class, since every state is
  reachable, because $S$ is canonical.
  Second, $\psi$ maps every state to at most one equivalence class, by
  Proposition~\ref{prop:statesim-implies-funcsim}.
  Third, $\psi$ is surjective since every equivalence class $[s]_F$ is assigned
  by $\psi$ to the the state that is reached by $s$.
  Fourth, $\psi$ is injective since there are no distinct states
  $x,x'$ such that such that $[s]_F = [s']_F$ for $s$ leading to $x$ and $s'$
  leading to $x'$.
  The equality $[s]_F = [s']_F$ holds only if $s \sim_F s'$, which holds only if  
  $F(sz) = F(s'z)$ for every $z$, which holds only if $S^x = S^{x'}$ since $S$
  implements $F$. Since $S$ is canonical, and hence in
  reduced form, we have that $S^x = S^{x'}$ does not hold, and hence the
  $\psi$ is injective.
  Therefore $\psi$ is a one-to-one correspondence as required.
\end{proof}

\begin{proposition} 
  \label{prop:equivalence-classes-for-connected-system}
  Consider a connected system $S$ on input domain $U$.
  The states of $S$ are in a one-to-one correspondence with the equivalence
  classes $U^*/{\sim_S}$.
  State $x$ corresponds to the equivalence class $[s]_S$ for $s$ any string that
  leads to $x$.
\end{proposition}
\begin{proof}
  Let $\psi$ be the mentioned correspondence.
  First, $\psi$ maps every state to some equivalence class, since every state is
  reachable, because $S$ is connected.
  Second, $\psi$ maps every state to at most one equivalence class, by
  the definition of the equivalence classes $X/{\sim_S}$.
  Third, $\psi$ is surjective since every equivalence class $[s]_S$ is assigned
  by $\psi$ to the state that is reached by $s$.
  Fourth, $\psi$ is injective since there are no distinct states $x,x'$ such
  that $[s]_S = [s']_S$ for $s$ leading to $x$ and $s'$ leading to $x'$.
  The equality $[s]_S = [s']_S$ holds only if $s \sim_S s'$, which holds only if  
  $s$ and $s'$ lead to the same state.
  Therefore $\psi$ is a one-to-one correspondence as required.
\end{proof}

\prophomomorphismconverse*
\begin{proof}
  Consider a system $S_1$ and a canonical system $S_2$.
  \begin{align*}
    S_1 & = \langle U, X_1, f_1, x_1^\mathrm{init}, Y, h_1 \rangle
    \\
    S_2 & = \langle U, X_2, f_2, x_2^\mathrm{init}, Y, h_2 \rangle
  \end{align*}
  Assume that the two systems are equivalent, i.e., they implement the same
  function $F$. 
  By Proposition~\ref{prop:equivalence-classes-for-canonical-system},
  every state of $S_2$ can be seen as an equivalence class $[w]_F$.
  Let $S_1'$ be the reachable subsystem of $S_1$, and let $D_1'$ be its
  dynamics.
  \begin{align*}
    D_1' & = \langle U, X_1', f_1' \rangle
  \end{align*}
  By Proposition~\ref{prop:equivalence-classes-for-connected-system}, every
  state of $S_1'$ can be seen as an equivalance class $[w]_{S_1}$.
  Let us define the function $\psi$ that maps $[w]_{S_1}$ to $[w]_F$.
  We have that $\psi$ is a well-defined function, i.e., it does not assign
  multiple values to the same input, by
  Proposition~\ref{prop:statesim-implies-funcsim} since $S_1$ implements $F$.

  \mytodo{double-check assumption on finite domain}
  We argue that $\psi$ is a continuous function as required by the definition of
  homomorphism.
  Assume by contradiction that $\psi$ is not continuous.
  Then, there exist $x = [w]_{S_1} \in X_1$ and $\epsilon > 0$ such that, for
  every $\delta > 0$, there exists $x' = [w']_{S_1} \in X_2$ such
  that\footnote{Note that $d_X$ denotes the metric function of a metric space
  $X$.}
  \begin{align*}
    d_{X_1}(x,x') < \delta
    \quad
    \text{and}
    \quad
    d_{X_2}(\psi(x),\psi(x')) \geq \epsilon.
  \end{align*}
  In particular, $d_{X_2}(\psi(x),\psi(x')) \geq \epsilon > 0$ implies 
  $\psi(x) \neq \psi(x')$.
  Let $z = \psi(x)$ and $z' = \psi(x')$.
  Since $S_2$ is canonical, there exists a string $u$ such that 
  $S_2^z(u) \neq S_2^{z'}(u)$.
  Let $u$ be the shortest such string, let $y$ be the last element in 
  $S_2^z(u)$ and let $y'$ be the last element in $S_2^{z'}(u)$.
  In particular, we have $y \neq y'$.
  Since $Y$ is finite or discrete, we have $d_Y(y,y') \geq \epsilon'$ for some
  $\epsilon' > 0$ which is independent of the choice of $\delta$ and $x'$.
  Therefore, we have shown that
  \begin{align*}
    d_{X_1}(x,x') < \delta
    \quad
    \text{and}
    \quad
    d_{Y}(y,y) \geq \epsilon'.
  \end{align*}
  Since $S_1$ and $S_2$ are equivalent, we have that 
  $y$ is also the last element of $S_1^x(u)$, and that
  $y'$ is also the last element of $S_1^{x'}(u)$.
  We now show that the outputs $y$ and $y'$ are obtained through a continuous
  function $g: X_1 \to Y$ of the state space of $S_1$.
  Let $u = a_1 \dots a_k$.
  Let $g_0(x) = x$ and let $g_i(x) = f_1(g_{i-1}(x), a_i)$ for $i \geq 1$.
  Then our desired function $g$ is defined as $g = h_1(g_{k-1}(x), a_k)$.
  We have that $g$ is continuous since it is the composition of continuous
  functions---in particular $f_1$ and $h_1$ are continuous by assumption.
  Finally, we have that $y = g(x)$ and $y' = g(x')$.
  Therefore, we have shown that
  \begin{align*}
    d_{X_1}(x,x') < \delta
    \quad
    \text{and}
    \quad
    d_{Y}(g(x),g(x')) \geq \epsilon'.
  \end{align*}
  Since $\delta$ can be chosen arbitrarily small, the former two conditions
  contradict the continuity of $g$. We conclude that $\psi$ is continuous.

  We argue that $\psi$ is a surjective function as required by the definition of
  homomorphism.
  The function is surjective since every state $q$ in $S_2$ is reachable, hence
  there is $w$ that reaches it, and hence $\psi$ maps $[w]_{S_1}$ to 
  $[w]_F = q$.

  Having argued the properties above, in order to show that $\psi$ is
  a homomorphism from $D_1'$ to $D_2$, it suffices to show that, for every $x
  \in X_1'$ and $u \in U$, the following equality holds.
  \begin{align*}
    \psi\big(f_1'(x,u)\big) = f_2\big(\psi(x),u\big)
  \end{align*}
  Let us consider arbitrary
  $x \in X_1'$ and $u \in U$.
  Let $w \in U^*$ be a string that reaches $x$ in $S_1'$.
  Note that $x$ can be seen as the equivalence class $[w]_{S_1}$.
  We have the following equivalences:
  \begin{align*}
    & \psi(f_1'(x,u)) = f_2(\psi(x),u)
    \\
    & \Leftrightarrow \psi(f_1'([w]_{S_1},u)) =
    f_2(\psi([w]_{S_1'}),u)
    \\
    & \Leftrightarrow \psi([wu]_{S_1}) = f_2(\psi([w]_{S_1}),u)
    \\
    & \Leftrightarrow [wu]_F = f_2(\psi([w]_{S_1}),u)
    \\
    & \Leftrightarrow [wu]_F = f_2([x]_F,u)
    \\
    & \Leftrightarrow [wu]_F = [wu]_F.
  \end{align*}
  The last equality holds trivially, and hence
  $\psi$ is a homomorphism from $D_1'$ to $D_2$. Since $D_1'$ is
  subdynamics of $D_1$, we conclude that $D_1$ homomorphically represents $D_2$.
  This proves the proposition.
\end{proof}

\section{Proofs for Section `Abstract
Neurons~\ref{sec:rncs_of_flipflop_neurons}'}

\subsection{Proof of Lemma~\ref{lemma:flipflop_to_sign_neuron}}

We prove Lemma~\ref{lemma:flipflop_to_sign_neuron} as two separate lemmas,
Lemma~\ref{lemma:flipflop_semiaton_to_flipflop_neuorn} for a flip-flop neuron 
and 
Lemma~\ref{lemma:G_semiaton_to_G_neuorn} for a $G$ neuron.
The proofs are nearly identical, but separate proofs are necessary.
We begin with the proof of 
Lemma~\ref{lemma:flipflop_semiaton_to_flipflop_neuorn}.

\begin{lemma}
    \label{lemma:flipflop_semiaton_to_flipflop_neuorn}
  Every flip-flop semiautomaton is homomorphically represented by
  a flip-flop neuron with arbitrary core.
\end{lemma}

\begin{proof}
  Let us consider a flip-flop semiautomaton 
  $A = \langle \Sigma, Q, \delta_\phi \rangle$, where $\phi : \Sigma \to \Pi$ is
  its input function, where $\Pi = \{ \mathit{read},\mathit{set},\mathit{reset}
  \}$.
  Let $N = \langle V, X, f \rangle$ be any core flip-flop neuron.
  Let $\epsilon$ be the minimum radius the intervals $\vread,\vset,\vreset$ of
  $N$.
  Note that $\epsilon > 0$ by definition.
  Let
  $v_\mathrm{read}$,
  $v_\mathrm{set}$, and 
  $v_\mathrm{reset}$
  be the midpoint of 
  $\vread,\vset,\vreset$, respectively.
  Note that
  $v_\mathrm{read} \pm \epsilon \in \vread$,
  $v_\mathrm{set} \pm \epsilon \in \vset$, and
  $v_\mathrm{reset} \pm \epsilon \in \vreset$.
  Let $\xi: \Pi \to \{ v_\mathrm{read},v_\mathrm{set},v_\mathrm{reset}, \}$ be
  defined as 
  \begin{align*}
    \xi(\pi) = 
    \begin{cases}
      v_\mathrm{reset} & \text{ if } \pi = \mathit{reset}
      \\
      v_\mathrm{read} & \text{ if } \pi = \mathit{read}
      \\
      v_\mathrm{set} & \text{ if } \pi = \mathit{set}
    \end{cases}
  \end{align*}
  Let us recall the input symbol grounding $\lambda_\Sigma: U \to \Sigma$.
  Let us define $\beta = \lambda_\Sigma \circ \phi \circ \xi$.
  Note that $\beta$ is a continuous function since $\lambda_\Sigma$ is
  continuous,
  and also $\phi$ and $\xi$ are continuous because they are functions with a
  finite domain.
  Since $\beta$ is continuous and $U$ is compact by assumption,
  by the universal approximation theorem for feedforward neural networks (cf.\
  Theorem~2 of
  \cite{hornik1991approximation}), there exists an $\epsilon$-approximation
  $\beta'$ of $\beta$.
  Note that $\phi(\lambda_\Sigma(u)) = \mathit{set}$ implies $\beta'(u) \in \vset$,
  $\phi(\lambda_\Sigma(u)) = \mathit{reset}$ implies $\beta'(u) \in \vreset$,
  $\phi(\lambda_\Sigma(u)) = \mathit{read}$ implies $\beta'(u) \in \vread$.

  Consider the flip-flop neuron $D = \langle U, X, f_{\beta'} \rangle$ obtained
  by composing $\beta'$ and $N$.
  
  First, we have that its state interpretation 
  $\psi : X \to Q$ with $Q = \{ \mathit{high}, \mathit{low} \}$ is a continuous
  function.
  To show this, it suffices to show that, for every $x_0 \in X$, and for every
  positive real number $\epsilon > 0$, there exists a positive real number 
  $\delta > 0$ such that all $x \in X$ satisfying $d_X(x,x_0) < \delta$ also
  satisfy $d_Q(\psi(x),\psi(x_0)) < \epsilon$.
  The state space $X$ consists of two disjoint closed intervals 
  $X_\mathrm{high}$ and $X_\mathrm{low}$.
  Let $\delta$ be the minimum distance between elements in the two intervals,
  i.e.,
  \begin{align*}
    \delta = \min\left\{ d_X(y,z) \mid y \in X_\mathrm{high}\text{, } z \in
    X_\mathrm{low}\right\}.
  \end{align*}
  Such a minimum distance exists because 
  $X_\mathrm{high}$ and $X_\mathrm{low}$ are
  disjoint closed intervals.
  Now assume $d_X(x,x_0) < \delta$. It follows that 
  $\{ x,x_0 \} \subseteq X_\mathrm{high}$ or 
  $\{ x,x_0 \} \subseteq X_\mathrm{low}$.
  In both cases, $\psi(x) = \psi(x_0)$ and hence 
  $d_Q(\psi(x),\psi(x_0)) = 0 < \epsilon$ for every $\epsilon > 0$, as required.

  Second, for arbitrary $x \in X$ and $u \in U$, we show
  \begin{align*}
    \psi(f(x,u)) = \delta_\phi(\psi(x),\lambda_\Sigma(u)).
  \end{align*}
  Three cases are possible according to the value of $\phi(\lambda_\Sigma(u))$.
  \begin{itemize}
    \item 
      In the first case $\phi(\lambda_\Sigma(u)) = \mathit{reset}$. 
      We have that
      \begin{align*}
        \delta_\phi(\psi(x),\lambda_\Sigma(u))
        = \delta(\psi(x), \phi(\lambda_\Sigma(u))) 
        = \delta(\psi(x), \mathit{reset})= \mathit{low}.
      \end{align*}
      Thus, it suffices to show $f(x,u) \in \xlow$.
      Since $\phi(\lambda_\Sigma(u)) = \mathit{reset}$,
      we have that
      $\beta'(u) \in \vreset$ as noted above, and hence the required condition
      holds by the definition of flip-flop neuron. 

    \item
      In the second case $\phi(\lambda_\Sigma(u)) = \mathit{read}$. 
      We have that 
      \begin{align*}
        \delta_\phi(\psi(x),\lambda_\Sigma(u)) 
        = \delta(\psi(x), \phi(\lambda_\Sigma(u)))
        = \delta(\psi(x), \mathit{read})= \psi(x).
      \end{align*}
      Thus, it suffices to show the two implications
      \begin{align*}
        x \in \xlow \quad & \Rightarrow \quad f(x,\lambda_\Sigma(u)) \in \xlow,
      \\
        x \in \xhigh \quad & \Rightarrow \quad f(x,\lambda_\Sigma(u)) \in \xhigh.
      \end{align*}
      Since $\phi(\lambda_\Sigma(u)) = \mathit{read}$,
      we have that $\beta'(u) \in \vread$ as noted above, and hence the required
      implications hold by the definition of flip-flop neuron.

    \item
      In the third case $\phi(\lambda_\Sigma(u)) = \mathit{set}$. 
      We have that 
      \begin{align*}
        \delta_\phi(\psi(x), \lambda_\Sigma(u))
        = \delta(\psi(x), \phi(\lambda_\Sigma(u)))
        = \delta(\psi(x), \mathit{set}) = \mathit{high}.
      \end{align*}
      Thus, it suffices to show $f(x,\lambda_\Sigma(u)) \in \xhigh$.
      Since $\phi(\lambda_\Sigma(u)) = \mathit{set}$,
      we have that $\beta'(u) \in \vset$ as noted above, and hence the required
      condition holds by the definition of flip-flop neuron. 
  \end{itemize}
  The lemma is proved.
\end{proof}

We next prove the Lemma~\ref{lemma:G_semiaton_to_G_neuorn}.

\begin{lemma}
    \label{lemma:G_semiaton_to_G_neuorn}
  Every $G$ semiautomaton is homomorphically represented by
  a $G$ neuron with arbitrary core.
\end{lemma}

\begin{proof}
 Let $G$ be a finite group.
  Let us consider a G semiautomaton 
  $A = \langle \Sigma, D, \delta_\phi \rangle$, where 
  $\phi : \Sigma \to D$ is its input function and 
  $D$ is the domain of G.
  Let $N = \langle V, X, f \rangle$ be any core $G$ neuron.
  Let $\epsilon$ be the minimum radius among intervals $V_i \in V$.
  Note that $\epsilon > 0$ by definition.
  Let $v_i$ be the midpoint of $V_i \in V$. 
  Note that $v_i \pm \epsilon \in V_i$ and 
  $v_i \pm \epsilon \notin V_j$ for $i \neq j$
  holds for every $V_i, V_j \in V$, since $V$ is a set of disjoint intervals.
  Let $\xi: D \to \{ v_i \mid V_i \in V \}$ be defined as 
  \begin{align*}
    \xi(D_i) = v_i.
  \end{align*}
  Let us recall the input symbol grounding $\lambda_\Sigma: U \to \Sigma$.
  Let us define $\beta = \lambda_\Sigma \circ \phi \circ \xi$.
  Note that $\beta$ is a continuous function since $\lambda_\Sigma$ is
  continuous,
  and also $\phi$ and $\xi$ are continuous because they are functions with a
  finite domain.
  Since $\beta$ is continuous and $U$ is compact by assumption,
  by the universal approximation theorem for feedforward neural networks (cf.\
  Theorem~2 of
  \cite{hornik1991approximation}), there exists an $\epsilon$-approximation
  $\beta'$ of $\beta$.
  Note that 
  $\phi(\lambda_\Sigma(u)) = D_i$ implies $\beta'(u) \in V_i$.

  Consider the $G$ neuron $D = \langle U, X, f_{\beta'} \rangle$ obtained
  by composing $\beta'$ and $N$.
  
  First, we have that its state interpretation 
  $\psi : X \to D$, with $D$ the domain of $G$, is a continuous function.
  To show this, it suffices to show that, 
  for every $x_0 \in X$, and for every
  positive real number $\epsilon > 0$, 
  there exists a positive real number $\delta > 0$ 
  such that all $x \in X$ satisfying 
  $d_X(x,x_0) < \delta$ 
  also satisfy $d_D(\psi(x),\psi(x_0)) < \epsilon$.
  Let $\delta$ be the minimum distance between any two intervals 
  $X_i, X_j \in X$ with $i \neq j$, i.e.,
  \begin{align*}
    \delta = \min\left\{ d_X(y,z) \mid 
             y \in X_i\text{, } z \in X_j\text{, for } 
             X_i, X_j \in X\text{ and }i \neq j
             \right\}.
  \end{align*}
  Such a minimum distance exists because every
  $X_i, X_j \in X$ for $i \neq j$ are disjoint closed intervals.
  Now assume $d_X(x,x_0) < \delta$. 
  It follows that 
  $\{ x,x_0 \} \subseteq X_i$ for some $X_i \in X$ therefore,
  $\psi(x) = \psi(x_0)$ and hence 
  $d_Q(\psi(x),\psi(x_0)) = 0 < \epsilon$ for every 
  $\epsilon > 0$, as required.

  Second, for arbitrary $x \in X$ and $u \in U$, we show that
  \begin{align*}
    \psi\big(f(x,u)\big) = \delta_\phi\big(\psi(x),\lambda_\Sigma(u)\big).
  \end{align*}
  We have that 
  \begin{align*}
    \delta_\phi\big(\psi(x),\lambda_\Sigma(u)\big)
    = \delta\big(\psi(x), \phi(\lambda_\Sigma(u))\big) 
    = \psi(x) \circ \phi(\lambda_\Sigma(u)).
  \end{align*}
  Let us denote by $i, j$ elements of $D$ such that 
  $\psi(x) = i$ and $\phi(\lambda_\Sigma(u)) = j$.
  Then, it suffices to show that $f(x,u) \in X_{i \circ j}$. 
  Since $\psi(x) = i$, we have that $x \in X_i$ as argued above. 
  Furthermore, since $\phi(\lambda_\Sigma(u)) = j$, we have that 
  $\beta'(u) \in V_j$ as argued above. Hence the required condition
   holds by the definition of $G$ neuron. Namely, for 
   $x \in X_i, u \in V_j$ we have that 
   \begin{align*}
   f(x,u) \in X_{i \circ j}.
   \end{align*}
  The lemma is proved.
\end{proof}

\subsection{%
Proof of Lemma~\ref{lemma:automaton_cascade_to_neural_cascade}
}

We prove Lemma~\ref{lemma:automaton_cascade_to_neural_cascade} as two separate
lemmas, 
Lemma~\ref{lemma:automaton_cascade_to_neural_cascade_p1}
for cascades
and 
Lemma~\ref{lemma:automaton_network_to_neural_network}
for networks.
The proofs are nearly identical, but separate proofs are necessary.
We begin with the proof of 
Lemma~\ref{lemma:automaton_cascade_to_neural_cascade_p1}.

\mytodo{double-check whether, for the composition, it is ok to have
  homomorphic representations, or we should switch to the stronger claim that
  they are actually homomorphic (i.e., no subset).}

\begin{lemma}
  \label{lemma:automaton_cascade_to_neural_cascade_p1}
  Every cascade of flip-flop or grouplike semiautomata
  $A_1, \dots, A_d$ is homomorphically represented
  by a cascade of d neurons
  $N_1, \dots, N_d$ where $N_i$ is a flip-flop neuron if $A_i$ is a 
  flip-flop semiautomaton and it is a $G$ neuron if $A_i$ is a 
  $G$ semiautomaton.
  The homomorphism is the cross-product of the state interpretations of the
  neurons.
\end{lemma}
\begin{proof}
  We have that the $j$-th semiautomaton of the cascade is the tuple
  \begin{align*}
    A_j = \langle (\Sigma \times Q^{j-1}), Q, \delta_{\phi_j} \rangle.
  \end{align*}
  Let $N_j$ be a core flip-flop neuron if $A_j$ is a flip-flop 
  semiautomaton 
  and let it be a core $G$ neuron of $A_j$ is a $G$ semiautomaton. 
  A neuron $N_j$ is a tuple 
  \begin{align*}
    N_j = 
    \langle 
    V_j,
    X_j,
    f_j
    \rangle.
  \end{align*}
  with state interpretation $\psi_j$.
  Note that $\lambda_\Sigma$ is the fixed input symbol grounding.
  Let us define 
  $\lambda_j = \lambda_\Sigma \times \psi_1 \times \dots \times \psi_{j-1}$.
  Note that $\lambda_j$ is a symbol grounding, from 
  $U \times X_1 \times \dots X_{j-1}$ to $\Sigma \times Q^{j-1}$.
  Under symbol grounding $\lambda_j$,
  by Lemma~\ref{lemma:flipflop_to_sign_neuron}, 
  component $A_j$ is homomorphically represented by a neuron $D_j$, with
  homomorphism the state interpretation $\psi_j$.
  \begin{align*}
    D_j = \langle 
    (U \times X_1 \times \dots \times X_{j-1}),
    X_j,
    f_{\beta_j}
    \rangle
  \end{align*}
  Let us define the RNC dynamics $D$ having the neurons above as components.
  \begin{align*}
    D = 
    \langle 
    U,
    (X_1 \times \dots X_d),
    f 
    \rangle.
  \end{align*}
  In order to prove the lemma, it suffices to show that 
  $\psi = \psi_1 \times \dots \times \psi_d$ is a homomorphism
  from $D$ to $A$, under the fixed symbol grounding $\lambda_\Sigma$.
  For that, it suffices to show the equality
  \begin{align*}
    \psi(f(\langle x_1, \dots, x_d \rangle,u)) = \delta(\psi(x_1, \dots,
    x_d),\lambda_\Sigma(u))
  \end{align*}
  for every $\langle x_1, \dots, x_d \rangle \in X^d$ and
  every $u \in U$.
  First, let us rewrite the equality
  \begin{align*}
    \psi(f(\langle x_1, \dots, x_d \rangle,u)) & = \delta(\psi(x_1,
    \dots, x_d),\lambda_\Sigma(u)),
    \\
    \psi(f(\langle x_1, \dots, x_d \rangle,u)) & = \delta(\langle 
    \psi_1(x_1), \dots, \psi_d(x_d) \rangle,\lambda_\Sigma(u)),
    \\
    \psi_i(f_i(x_i, \langle u, x_1, \dots, x_{i-1} \rangle )) & =
    \delta_{\phi_i}(\psi_i(x_i), \langle \lambda_\Sigma(u) , \psi_1(x_1), \dots,
    \psi_{i-1}(x_{i-1}) \rangle)) \quad \forall\, 1 \leq i \leq d.
    \\
    \psi_i(f_i(x_i, \langle u, x_1, \dots, x_{i-1} \rangle )) & =
    \delta_{\phi_i}(\psi_i(x_i), \lambda_i(u, x_1, \dots,
    x_{i-1})) \quad \forall\, 1 \leq i \leq d.
  \end{align*}
  Then, the equalities hold since component $A_j$ is
  homomorphically represented by the neuron $D_j$, under $\lambda_j$, 
  with homomorphism $\psi_j$.
  The lemma is proved.
\end{proof}

We next prove Lemma~\ref{lemma:automaton_network_to_neural_network}.

\begin{lemma}
  \label{lemma:automaton_network_to_neural_network}
  Every network of flip-flop or grouplike semiautomata
  $A_1, \dots, A_d$ is homomorphically represented
  by a network of d neurons
  $N_1, \dots, N_d$ where $N_i$ is a flip-flop neuron if $A_i$ is a 
  flip-flop semiautomaton and it is a $G$ neuron if $A_i$ is a 
  $G$ semiautomaton.
\end{lemma}
\begin{proof}
  We have that the $j$-th component of the network is a semiautomaton is the
  tuple
  \begin{align*}
    A_j = \langle (\Sigma \times Q^{d-1}), Q, \delta_{\phi_j} \rangle.
  \end{align*}
  Let $N_j$ be a core flip-flop neuron if $A_j$ is a flip-flop semiautomaton
  and let it be a core $G$ neuron if $A_j$ is $G$ semiautomaton. 
  A neuron $N_j$ is a tuple 
  \begin{align*}
    N_j = 
    \langle 
    V_j,
    X_j,
    f_j
    \rangle
  \end{align*}
  with state interpretation $\psi_j$.
  Note that $\lambda_\Sigma$ is the fixed input symbol grounding.
  Let us define 
  $\lambda_j = \lambda_\Sigma \times \psi_1 \times \dots \times \psi_{j-1} \times
  \psi_{j+1} \times \dots \times \psi_d$.
  Note that $\lambda_j$ is a symbol grounding, from 
  $U \times X_1 \times \dots X_{j-1} \times X_{j+1} \times \dots \times X_d$ to
  $\Sigma \times Q^{d-1}$.
  Under symbol grounding $\lambda_j$,
  by Lemma~\ref{lemma:flipflop_to_sign_neuron}, 
  component $A_j$ is homomorphically represented by a neuron $D_j$ with core
  $N_j$, with homomorphism the state interpretation $\psi_j$.
  \begin{align*}
    D_j = \langle 
    (U \times X_1 \times \dots X_{j-1} \times X_{j+1} \times \dots \times X_d),
    X_j,
    f_{\beta_j}
    \rangle
  \end{align*}
  Let us define the network $D$ having neurons $D_1, \dots, D_d$ as
  components.
  \begin{align*}
    D = 
    \langle 
    U,
    (X_1 \times \dots \times X_d),
    f 
    \rangle
  \end{align*}
  In order to prove the lemma, it suffices to show that
  $\psi$ is a homomorphism from $D$ to $A$, under the fixed symbol grounding
  $\lambda_\Sigma$.
  For that, it suffices to show the equality
  \begin{align*}
    \psi(f(\langle x_1, \dots, x_d \rangle,u)) = \delta(\psi(x_1, \dots,
    x_d),\lambda_\Sigma(u))
  \end{align*}
  for every $\langle x_1, \dots, x_d \rangle \in X^d$ and
  every $u \in U$.
  Let us rewrite the condition
  \begin{align*}
    \psi(f(\langle x_1, \dots, x_d \rangle,u)) & = \delta(\psi(x_1,
    \dots, x_d),\lambda_\Sigma(u)),
    \\
    \psi(f(\langle x_1, \dots, x_d \rangle,u)) & = \delta(\langle 
    \psi_1(x_1), \dots, \psi_d(x_d) \rangle,\lambda_\Sigma(u)),
  \end{align*}
  which is equivalent to requiring that, for every $1 \leq i \leq d$, 
  \begin{align*}
    & \psi_i(f_i(x_i, \langle u, x_1, \dots, x_{i-1}, x_{i+1}, \dots, x_d \rangle
    )) 
    \\
    & = 
    \delta_{\phi_i}(\psi_i(x_i), \langle \lambda_\Sigma(u) , \psi_1(x_1), \dots,
    \psi_{i-1}(x_{i-1}),\psi_{i+1}(x_{i+1}), \dots,\psi_d(x_d) \rangle)) 
  \end{align*}
  which is in turn equivalent to
  \begin{align*}
    \psi_i(f_i(x_i, \langle u, x_1, \dots, x_{i-1},x_{i+1},\dots,x_d \rangle ))
    &
    = \delta_{\phi_i}(\psi_i(x_i), \lambda_i(u, x_1, \dots,
    x_{i-1},x_{i+1},\dots,x_d )).
  \end{align*}
  Each of the equalities holds since component $A_i$ is
  homomorphically represented by the neuron $D_i$, under symbol grounding
  $\lambda_i$, with homomorphism $\psi_i$.
  This lemma is proved.
\end{proof}

\section{Proofs for Section `Implementation of Flip-Flop Neurons'}

\subsection{Proof of Proposition~\ref{prop:sign_transitions}}

We prove Proposition~\ref{prop:sign_transitions}.

\propcoreflipflopsign*
\begin{proof}
    We show that the neuron with states 
    $X = \xlow \cup \xhigh$, 
    inputs $V = \vset \cup \vreset \cup \vread$, 
    weight $w > 0$ and sign activation function satisfies the conditions
    of the definition of a core flipflop neuron. 
    First we show that conditions 
  \begin{align*}
      f(X,\vset) \subseteq \xhigh,
      \;
      f(X,\vreset) \subseteq \xlow,
      \;
      f(\xhigh,\vread) \subseteq \xhigh,
      \;
      f(\xlow,\vread) \subseteq \xlow,
  \end{align*}
  are satisfied by 
  \begin{align*}
    f(x,v) = \sign(w \cdot x + v),
  \end{align*}
  with $w > 0$, $x \in X$, and $v \in V$.

  Condition $f(X,\vset) \subseteq \xhigh$ is satisfied.
  By the premise of the proposition for $x \in X$ and $v \in \vset$ 
  it holds that
  $x \geq -1$ and $v \geq w \cdot (a + 1)$.
  Furthermore, $w \cdot a > 0$ since $w > 0$ and $a > 0$. Thus,
  \begin{align*}
    w \cdot x + v 
    \geq -w + v 
    \geq  -w + w \cdot (a + 1) 
    = -w + w \cdot a + w
    = w \cdot a
    > 0, 
  \end{align*}
  and hence,
  \begin{align*}
    \sign(w \cdot x + v) = +1 \in \xhigh.
  \end{align*}

  Condition $f(X,\vreset) \subseteq \xlow$ is satisfied.
  By the premise of the proposition for $x \in X$ and $v \in \vreset$
  it holds that
  $x \leq +1$ and $v \leq w \cdot (-a - 1)$.
  Furthermore, $w \cdot (-a) < 0$ since $w > 0$ and $a > 0$.
  Thus,
  \begin{align*}
    w \cdot x + v 
    \leq  w + v 
    \leq w + w \cdot (-a - 1)
    = w + w \cdot (-a) + (-w) 
    = w \cdot (-a)
    < 0, 
  \end{align*}
  and hence,
  \begin{align*}
    \sign(w \cdot x + v) = -1 \in \xlow.
  \end{align*}

  Condition $f(\xhigh,\vread) \subseteq \xhigh$ is satisfied.
  By the premise of the proposition for $x \in \xhigh$ and $v \in \vread$
  it holds that 
  $x = +1$ and $v \geq w \cdot (a -1)$.
  Furthermore, $w \cdot a > 0$ since $w > 0$ and $a > 0$.
  Thus,
  \begin{align*}
    w \cdot x + v 
    = w + v 
    \geq w + w \cdot (a - 1)
    = w + w \cdot a + (-w)
    = w \cdot a
    > 0, 
  \end{align*}
  and hence,
  \begin{align*}
    \sign(w \cdot +1 + v) = +1 \in \xhigh.
  \end{align*}

  Condition $f(\xlow,\vread) \subseteq \xlow$ is satisfied.
  By the premise of the proposition for $x \in \xlow$ and $v \in \vread$
  it holds that
  $x = -1$ and $v \leq w \cdot (1 - a)$. 
  Furthermore, $w \cdot (-a) < 0$ since $w > 0$ and $a > 0$.
  Thus,
  \begin{align*}
    w \cdot -1 + v 
    = -w + v 
    \leq -w + w \cdot (1 - a)
    = -w + w  + w \cdot (-a)
    = w \cdot (-a)
    < 0, 
  \end{align*}
  and hence,
  \begin{align*}
    \sign(w \cdot -1 + v) = +1 \in \xlow.
  \end{align*}

  Finally, we show that the set $X$ of states and the set $V$ of inputs 
  satisfy the conditions of the definition of a core flip-flop neuron.

  The set $X$ is the union of two disjoint closed intervals by 
  definition, since 
  since $\xlow = \{-1\}$ and $\xhigh = \{+1\}$.

  The set $V$ is the union of three disjoint closed intervals of non-zero 
  length. 
  By the premise of the proposition  
  the interval 
  $\vset$ is bounded only on the left by a number $w \cdot (1 - a)$,
  the interval
  $\vreset$ is bounded only on the right by a number $w \cdot (-a -1)$, and
  the interval $\vread$ is bounded by numbers 
  $w \cdot (a -1)$ and $w \cdot (1 - a)$ on the left and on the
  right respectively,
  hence all three interval are closed by definition.
  The intervals $\vset$ and $\vreset$ are of non-zero length by definition.
  We also have that $w \cdot (a -1) < 0$ and $w \cdot (1 - a) > 0$, since
  $w > 0$ and $0 < a < 1$. Therefore, the interval $\vread$ is also of
  non-zero length.
  Furthermore, the intervals $\vset$, $\vread$, and $\vreset$ are disjoint,
  since for $w > 0$ and $a > 0$ the following inequalities 
  \begin{align*}
      w\cdot (-a-1) < -w < w \cdot (a -1) 
      < 
      w \cdot (1 - a) < w < w \cdot (a + 1).
  \end{align*}
  This proposition is proved.
\end{proof}

\subsection{Proof for Proposition~\ref{prop:tanh_transitions}}

We prove Proposition~\ref{prop:tanh_transitions}.

\propcoreflipfloptanh*
\begin{proof}
    We show that a neuron with states 
    $X = \xlow \cup \xhigh$, 
    inputs $V = \vset \cup \vreset \cup \vread$, 
    weight $w > 1$ and tanh activation function satisfies the conditions
    of the definition of a core flipflop neuron. 
    First we show that conditions 
    \begin{align*}
      f(X,\vset) \subseteq \xhigh,
      \;
      f(X,\vreset) \subseteq \xlow,
      \;
      f(\xhigh,\vread) \subseteq \xhigh,
      \;
      f(\xlow,\vread) \subseteq \xlow,
  \end{align*}
  are satisfied by 
  \begin{align*}
    f(x,v) = \tanh(w \cdot x + v),
  \end{align*}
  with $w > 1$, $x \in X$ and $v \in V$.

  Condition $f(X,\vset) \subseteq \xhigh$ is satisfied.
  By the premise of the proposition for $x \in X$ and $v \in \vset$ 
  it holds that
  $x \geq -1$ and $v \geq w \cdot (b + 1)$. Thus,
  \begin{align*}
    w \cdot x + v 
    \geq -w + v
    \geq -w + w \cdot (b + 1)
    = -w + w \cdot b + w
    = w \cdot b.
  \end{align*}
  Since tanh is monotonic it holds that
  $\tanh(w \cdot x + v) \geq \tanh(w \cdot b) = f(b)$, and hence
  \begin{align*}
      \tanh(w \cdot x + v) \in \xhigh.
  \end{align*}

  Condition $f(X,\vreset) \subseteq \xlow$ is satisfied.
  By the premise of the proposition for $x \in X$ and $v \in \vreset$
  it holds that
  $x \leq +1$ and $v \leq w \cdot (a - 1)$. Thus,
  \begin{align*}
    w \cdot x + v 
    \leq w + v
    \leq w + w \cdot (a - 1)
    = w + w \cdot a + (-w)
    = w \cdot a.
  \end{align*}
  Since tanh is monotonic it holds that
  $\tanh(w \cdot x + v) \leq \tanh(w \cdot a) = f(a)$, and hence
  \begin{align*}
      \tanh(w \cdot x + v) \in \xlow.
  \end{align*}

  Condition $f(\xhigh,\vread) \subseteq \xhigh$ is satisfied.
  By the premise of the proposition for $x \in \xhigh$ and $v \in \vread$
  it holds that
  $x \geq f(b)$ and $v \geq w \cdot (b - f(b))$. Thus
  \begin{align*}
    w \cdot x + v 
    \geq w \cdot f(b) + v
    \geq w \cdot f(b) + w \cdot (b - f(b))
    = w \cdot f(b) + w \cdot b + w \cdot  (-f(b))
    = w \cdot b
  \end{align*}
  Since tanh is monotonic it holds that
  $\tanh(w \cdot x + v) \geq \tanh(w \cdot b) = f(b)$, and hence
  \begin{align*}
      \tanh(w \cdot x + v) \in \xhigh.
  \end{align*}

  Condition $f(\xlow,\vread) \subseteq \xlow$ is satisfied.
  By the premise of the proposition for $x \in \xlow$ and $v \in \vread$
  it holds that
  $x \leq f(a)$ and $v \leq w \cdot (a - f(a))$. Thus
  \begin{align*}
    w \cdot x + v 
    \leq w \cdot f(a) + v
    \leq w \cdot f(a) + w \cdot (a - f(a))
    = w \cdot f(a) + w \cdot a + w \cdot  (-f(a))
    = w \cdot a
  \end{align*}
  Since tanh is monotonic it holds that
  $\tanh(w \cdot x + v) \leq \tanh(w \cdot a) = f(a)$, and hence
  \begin{align*}
      \tanh(w \cdot x + v) \in \xlow.
  \end{align*}

  Finally, we show that the set $X$ of states and the set $V$ of inputs 
  satisfy the conditions of the definition of a core flip-flop neuron.

  The set $X$ is the union of two disjoint closed intervals. 
  Sets $\xlow = [-1, f(a)]$ and $\xhigh = [f(b), +1]$ are closed by 
  definition, since they are bounded by a number on both sides.
  By the premise of the proposition, it holds that $a < b, w > 1$ and thus 
  $w \cdot a < w \cdot b$. Since $\tanh$ is monotonic it holds that
  $f(a) < f(b)$, therefore $\xlow$ and $\xhigh$ are disjoint.

  The set $V$ is the union of three disjoint closed intervals of non-zero 
  length. 
  By the premise of the proposition 
  the interval 
  $\vreset$ is bounded only on the right by a number $w \cdot (a - 1)$,
  the interval 
  $\vset$ is bounded only on the left by a number $w \cdot (b + 1)$, and
  the interval $\vread$ is bounded by numbers $w \cdot (b - f(b))$ and
  $w \cdot (a - f(a))]$ on the left and on the right respectively,
  hence all three intervals are closed by definition.

  The intervals $\vreset$ and $\vset$ are of non-zero length by definition.

  Since $w$ is positive, in order to show that $\vread$ is of non-zero 
  length it suffices to show that numbers $a < b$ satisfying 
  $a - f(a) > b - f(b)$ exist.

  Consider derivative $\tanh'(w \cdot x)$ for $w > 1$. 
  It is a bell-shaped function reaching maximum value of $w$ at $x = 0$, 
  with tails approaching zero for values of $x$ going towards $\pm\infty$
  Let $g(x)  = x - \tanh(w \cdot x)$. 
  It is derivative is $g'(x) = 1 - \tanh'(w \cdot x)$. 
  Then $g'(x) = 1 - w < 0$ for $x = 0$ since $w > 1$, and
  $g'(x)$ is approaching 1 for values of $x$ going towards $\pm\infty$ 
  because $\tanh'(w \cdot x)$ is approaching zero.
  Then, there exist two points $-p < 0 < p$ where $g'(x)$ crosses the 
  $x$-axis and thus $g'(x) < 0$ for $-p < x < p$. Thus, function
  $g(x)$ is decreasing between $-p$ and $p$. 
  Therefore, two numbers $a,b$ can be found such that 
  $a < b$ and $a - f(a) > b - f(b)$.

  The $\vreset$ and $\vread$ are disjoint, since by the premise of the 
  proposition $a < b$ and $w > 1$ and hence $w \cdot (a-1) < w \cdot (b+1)$.
  By the property of tanh we have that $f(a), f(b) < 1$, then 
  the following inequalities hold
  \begin{align*}
      w \cdot a - w &< w \cdot b - w \cdot f(b),
\\
      w \cdot a - w \cdot f(a) &< w \cdot b + w.
  \end{align*}
  Hence $\vreset$ and $\vread$ are disjoint and $\vread$ and $\vset$ are 
  disjoint.
  The proposition is proved.
\end{proof}

\section{Proofs for Section `Expressivity of RNCs'}

\subsection{Proof of Theorem~\ref{theorem:expressivity_of_rnc}}
We prove Theorem~\ref{theorem:expressivity_of_rnc}.

\theoremexpressivityofrnc*
\begin{proof}
  Consider a group-free regular function $F$. 
  There is an automaton that implements $F$ and has a group-free semiautomaton
  $D$.
  By Theorem~\ref{theorem:krohn_rhodes}, semiautomaton $D$
  is homomorphically represented by a cascade $C$ of flip-flop semiautomata.
  By Lemma~\ref{lemma:automaton_cascade_to_neural_cascade},
  $C$ is homomorphically represented by a cascade $C'$ of flip-flop neurons with
  arbitrary core.
  By Proposition~\ref{prop:homomorphism},
  there is an RNC $S$ with dynamics $C'$ that implements the same function as 
  the system $A_{\lambda_\Sigma}$, obtained as the composition of $A$ with the
  fixed symbol grounding $\lambda_{\Sigma}$.
  The output function of $S$ can be $\epsilon$-approximated by a feedforward
  neural network $h$, by the universal approximation theorem, cf.\
  \cite{hornik1991approximation}.
  Thus, the RNC $N$ is obtained from $S$ by replacing its output
  function with the feedforward neural network $h$.
  Since the output symbol grounding $\lambda_\Gamma$ is $\epsilon$-robust,
  the system $N_{\lambda_\Gamma}$ obtained by the composition of $N$ with the
  output symbol grounding $\lambda_\Gamma$ is equivalent to
  $A_{\lambda_\Sigma}$. Since $F$ is an arbitrary group-free function, the
  theorem is proved. \mytodo{instantiate for sign and tanh}
\end{proof}

\subsection{Proof of Lemma~\ref{lemma:permutation_alternating_states}}

The proof of Lemma~\ref{lemma:permutation_alternating_states} requires two
preliminary results, Theorem~\ref{theorem:ginzburg-subgroup} that we state
and 
Lemma~\ref{lemma:nongroupfree-has-permutation} that we state and prove.
Then we prove Lemma~\ref{lemma:permutation_alternating_states}.

%
\begin{theorem}[Theorem~A in 1.16 from Ginzburg]
  \label{theorem:ginzburg-subgroup}
  For every homomorphism $\psi$ of a finite semigroup $S$ onto a group $G$,
  there exists a subgroup (i.e., a subsemigroup which is a group) $K$ of $S$
  such that $\psi(K) = G$.
\end{theorem}

\begin{lemma}
  \label{lemma:nongroupfree-has-permutation}
  If a semiautomaton is not group-free, then it has a subsemiautomaton with an
  input that induces a non-identity permutation.
\end{lemma}
\begin{proof}
  Consider an automaton $A$ and let us assume that $A$ is not group-free.
  Let $S_A$ be the characteristic semigroup of $A$.
  By the definition of group-free, we have that $S_A$ has a subsemigroup $S'$
  that is divided by a nontrivial group $G$.
  Namely, there exists a homomorphism of $S'$ onto $G$.
  Then, by Theorem~\ref{theorem:ginzburg-subgroup}, there exists a subgroup $K$
  of $S'$.
  Consider the subautomaton $A'$ of $A$ whose characteristic semigroup is $G$.

  Since $G$ is nontrivial, there it has a transformation $p$ that is not an
  identity.
  Since $p$ is invertible, hence injective, it is also surjective by the
  pigeonhole principle. Thus $p$ is bijective, i.e., a permutation of the states.

  Say that $p$ is given by the composition $\tau_1 \circ \dots \circ \tau_k$
  of transformations induced by inputs of $A'$.

  First, at least one $\tau_i$ is not identity, otherwise $p$ would be identity.

  We show by induction on $k$ that $\tau_1, \dots, \tau_k$ are permutations.
  In the base case $k=1$. Then $p = \tau_1$ which is a permutation.
  In the inductive case $k \geq 2$ and we assume that 
  $\tau_1, \dots, \tau_{k-1}$ are permutations.
  It follows that $\tau_k$ is a permutation, otherwise there exist $x,x'$ such
  that $\tau_k(x) = \tau_k(x')$ and there exist $y,y'$ such that
  $(\tau_1 \circ \dots \circ \tau_{k-1})(y) = x$
  and
  $(\tau_1 \circ \dots \circ \tau_{k-1})(y') = x'$,
  and hence $p(y) = p(y')$, contradicting the fact that $p$ is a permutation.
\end{proof}

\lemmapermutationalternatingstates*
\begin{proof}
  Consider a semiautomaton $A = \langle \Sigma, Q, \delta \rangle$.
  By Lemma~\ref{lemma:nongroupfree-has-permutation},
  semiautomaton $A$ has a subsemiautomaton 
  $A' = \langle \Sigma', Q', \delta' \rangle$ such that it has an input $\sigma
  \in \Sigma'$ that induces a non-identity
  permutation transformation $\tau : Q' \to Q'$.
  Then, $\Sigma \subseteq \Sigma'$, hence $\sigma$ is also an input of $A$, and
  hence it induces a transformation $\tau$ of $A$ that non-identically permutes
  the subset $Q'$ of its states $Q$.
  Choosing $q_0 \in Q'$ and defining $q_i = \tau(q_{i-1})$ for 
  $i \geq 1$, we have that $q_i \neq q_{i+1}$ for every $i \geq 0$---we will use
  this fact in the last step of the proof.
  Now, let $x_0 \in X$ be such that $\psi(x_0) = q_0$, and let 
  $u \in U$ be such that $\lambda_\Sigma(u) = \sigma$. They exist because $\psi$ and
  $\lambda_\Sigma$ are surjective.
  Let us define the sequence $x_i = f(x_{i-1},u)$ for $i \geq 1$.

  We show by induction that $q_i = \psi(x_i)$ for every $i \geq 0$.
  In the base case $i = 0$, and we have $q_0 = \psi(x_0)$ by construction.
  In the inductive case $i > 0$, we assume 
  $q_{i-1} = \psi(x_{i-1})$, and we have
  $\psi(x_i) = \psi(f(x_{i-1},u)) = \delta(\psi(x_{i-1},\lambda_\Sigma(u))) =
  \delta(q_{i-1},\lambda_\Sigma(u))) = q_i$.

  We conclude that $\psi(x_i) = q_i \neq q_{i+1} = \psi(x_{i+1})$ for every 
  $i \geq 0$, as required. The lemma is proved.
\end{proof}

\subsection{Proof of Lemma~\ref{lemma:groupfree}}

The proof of Lemma~\ref{lemma:groupfree} requires two preliminary lemmas, 
Lemma~\ref{lemma:convergence_of_iterated_sign} and
Lemma~\ref{lemma:convergence_of_iterated_activation} that we state
and prove below. We first state and
prove
Lemma~\ref{lemma:convergence_of_iterated_sign}.

\begin{lemma}
  \label{lemma:convergence_of_iterated_sign}
  Let $w > 0$, let $(u_n)_{n \geq 1}$ be a real-valued sequence, 
  let $x_0 \in [-1, +1]$, and let 
  $x_n = \sign(w \cdot x_{n-1} + u_n)$ for $n \geq 1$.
  If the sequence $(u_n)_{n \geq 1}$ is convergent,
  then the sequence $(x_n)_{n \geq 0}$ is convergent.
\end{lemma}
\begin{proof}
  Since $(u_n)_{n \geq 1}$ is convergent, there exist a real number $u_*$, a
  natural number $N$, and a real number $0 < \epsilon < w$ such that, for every
  $n \geq N$, it holds that
    \begin{align*}
        |u_n - u_*| < \epsilon.
    \end{align*}
    According to the value of $u_*$ and $0 < \epsilon < w$, three cases are
    possible
    \begin{align*}
        u_* + \epsilon < -w \text{, }\mbox{ }
        -w \leq u_* - \epsilon, u_* + \epsilon < w \text{, }\mbox{ }
        u_* - \epsilon \geq w.
    \end{align*}
    We consider three cases separately and show that in each of them the 
    sequence $(x_n)_{n \geq N}$ is constant, hence the sequence
    $(x_n)_{n \geq 0}$ is trivially convergent.

    Note that the sign of zero is different according to different conventions.
    Here we consider $\sign(0) = +1$. 

    In the first case $u_* + \epsilon < -w$. 
    Let $n \geq N$, we have that
    \begin{align*}
        w \cdot x_{n-1} + u_n
        \leq 
        w \cdot x_{n-1} + u_* + \epsilon
        \leq 
        w + u_* + \epsilon
        < 0. 
    \end{align*}
    In particular, the inequality above holds for every 
    $x_{n-1} \in [-1, +1]$. 
    It follows that for every $n \geq N$
    \begin{align*}
        x_n = \sign(w \cdot x_{n-1} + u_n) = -1.
    \end{align*}

    In the second case $u_* - \epsilon \geq w$. 
    Let $n \geq N$, we have that
    \begin{align*}
        w \cdot x_{n-1} + u_n
        \geq
        w \cdot x_{n-1} + u_* - \epsilon
        \geq 
        -w + u_* - \epsilon
        \geq 0.
    \end{align*}
    In particular, the inequality above holds for every $x_{n-1} \in [-1, +1]$.
    It follows that for every $n \geq N$
    \begin{align*}
        x_n = \sign(w \cdot x_{n-1} + u_n) = +1.
    \end{align*}

    In the third case $-w \leq u_* - \epsilon, u_* + \epsilon < w$. 
    We show by induction that the sequence $(x_n)_{n \geq N}$ is constant.
    In the base case $n = N$, the sequence contains one element and is 
    therefore constant. 
    In the inductive step, assume that $x_n = x_{n-1}$. 
    If $x_{n} = -1$, then 
    \begin{align*}
        w \cdot x_{n} + u_{n+1}
        = 
        -w + u_{n+1}
        \leq 
        -w + u_* + \epsilon
        < 0,
    \end{align*}
    and thus, 
    \begin{align*}
        x_{n+1} = \sign(w \cdot x_{n} + u_{n+1}) = -1 = x_{n}
    \end{align*}
    If, however, $x_{n} = +1$, then 
    \begin{align*}
        w \cdot x_{n} + u_{n+1}
        = w + u_{n+1}
        \geq 
        w + u_* - \epsilon 
        \geq 0,
    \end{align*}
    and thus, 
    \begin{align*}
        x_{n+1} = \sign(w \cdot x_{n} + u_{n+1}) = +1 = x_{n}.
    \end{align*}
    Thus, for every $n \geq N$ it holds that $x_n = x_{n-1}$ and therefore,
    the sequence $(x_n)_{n\geq N}$ is constant.
    This lemma is proved.
\end{proof}

Next, we state and prove 
Lemma~\ref{lemma:convergence_of_iterated_activation}.

\begin{lemma}
  \label{lemma:convergence_of_iterated_activation}
  Let $\alpha: \mathbb{R} \to \mathbb{R}$ be a bounded, monotone, 
  Lipschitz continuous function.
  Let $w \geq 0$ be a real number, let $(u_n)_{n \geq 1}$ be a sequence of real
  numbers, let $x_0 \in \mathbb{R}$, and let 
  $x_n = \alpha(w \cdot x_{n-1} + u_n)$ for $n \geq 1$.
  If the sequence $(u_n)_{n \geq 1}$ is convergent,
  then the sequence $(x_n)_{n \geq 0}$ is convergent.
\end{lemma}
\begin{proof}
    Since $(u_n)_{n \geq 1}$ is convergent, there exists a real number $u_*$ 
    and a non-increasing sequence $(\epsilon_n)_{n \geq 1}$ that converges to 
    zero
    such that for every $n \geq 1$, it holds that
    \begin{align*}
        |u_n - u_*| < \epsilon_n.
    \end{align*}
    Thus, for every $n \geq 1$, it holds that 
    \begin{align*}
        u_n \in [u_*-\epsilon_n, u_*+\epsilon_n].
    \end{align*}
    Then, considering that function $\alpha$ is monotone, every element 
    $x_n = \alpha(w \cdot x_{n-1}+u_n)$ of the sequence
    $(x_n)_{n \geq 1}$ is bounded as
    \begin{align*}
        x_n &\in
        [
            \alpha(w \cdot x_{n-1} + u_*-\epsilon_n),
            \alpha(w \cdot x_{n-1} + u_*+\epsilon_n)
        ].
    \end{align*}
    Let us denote by $\ell_n$ and $r_n$ the left and right boundary for 
    the element $x_n$, 
    \begin{align*}
        x_n \in [\ell_n, r_n].
    \end{align*}
    In order to show that the sequence $(x_n)_{n \geq 1}$ is convergent, 
    it suffices to 
    show that the sequence $(\ell_n)_{n\geq 1}$ is convergent.  
    Recall that by the assumption of the lemma, 
    $\alpha$ is Lipschitz continuous, thus
    the sequence $(|\ell_n - r_n|)_{n \geq 1}$ tends to zero as the
    sequence $(u_*\pm\epsilon_n)_{n \geq 1}$ tends to $u_*$. 
    Therefore, if $(\ell_n)_{n\geq 1}$ is converging to some limit $\ell_*$, 
    the sequence $(r_n)_{n\geq 1}$ is converging to the same limit.
    Then by the Squeeze Theorem, $(x_n)_{n\geq 1}$ is convergent, since 
    $(\ell_n)_{n\geq 1}$ and $(r_n)_{n \geq 1}$ are converging to the same 
    limit and the relation $\ell_n \leq x_n \leq r_n$ holds for every 
    $n \geq 1$.

    We next show that the sequence $(\ell_n)_{n\geq 1}$ is indeed convergent.
    In particular, we show that there exists a $k$ such that the sequence 
    $\ell_{n \geq k}$ is monotone. Then considering that $\alpha$ is bounded, 
    we conclude that the sequence $\ell_{n \geq k}$ is convergent by the 
    Monotone Convergence Theorem.

    Consider the sequence $(u_*-\epsilon_n)_{n \geq 1}$. 
    We have that $(\epsilon_n)_{n \geq 1}$ is non-increasing, then 
    $(u_*-\epsilon_n)_{n \geq 1}$ is either non-decreasing or 
    is non-increasing.
    We consider the case when $(u_*-\epsilon_n)_{n \geq 1}$ is non-decreasing. 
    Then the case when it is non-increasing follows by symmetry.

    Let us denote by $u^*_n = u_*-\epsilon_n$ and 
    let us define the following quantities
    \begin{align*}
        \delta_{u^*_n}  = u^*_{n+1} - u^*_n\text{ and }
        \delta_{\ell_n} = \ell_{n} - \ell_{n-1}. 
    \end{align*}
    Two cases are possible, either 
    $w \cdot \delta_{\ell_n} + \delta_{u^*_n} < 0$
    for every $n \geq 1$, or there exists $n_0$ such that 
    $w \cdot \delta_{\ell_{n_0}} + \delta_{u^*_{n_0}} \geq 0$.

    Consider the case where $w \cdot \delta_{\ell_n} + \delta_{u^*_n} < 0$
    holds for every $n \geq 1$. We show that the sequence 
    $(\ell_n)_{n \geq 1}$ is non-increasing. 
    Since sequence 
    $(u^*_n)_{n \geq 1}$ is non-decreasing, then for every
    $n \geq 0$ it holds that 
    $\delta_{u^*_n}  = u^*_{n+1} - u^*_n \geq 0$. 
    Then, in order to satisfy the inequality
    $w \cdot \delta_{\ell_{n}} + \delta_{u^*_{n}} < 0$ it must hold that
    $\delta_{\ell_n} = \ell_{n} - \ell_{n-1} < 0$.
    This implies that $\ell_{n-1} \geq \ell_n$ holds for every $n \geq 1$
    and the sequence 
    $(\ell_n)_{n \geq 1}$ is non-increasing and thus monotone.

    We now consider the case where there exists an $n_0$, such that 
    $w \cdot \delta_{\ell_{n_0}} + \delta_{u^*_{n_0}} \geq 0$.
    First, we show by induction that the inequality 
    $w \cdot \delta_{\ell_{n}} + \delta_{u^*_{n}} \geq 0$ holds
    for every $n \geq n_0$.

    In the base case, $n = n_0$ and the inequality holds. 
    In the inductive step let us assume that inequality
    $w \cdot \delta_{\ell_{n}} + \delta_{u^*_{n}} \geq 0$ holds.
    We have that
    \begin{align*}
        w \cdot \ell_{n-1} + w \cdot \delta_{\ell_n}
        &= 
        w \cdot \ell_n
    \\
        w \cdot \ell_{n-1} + w \cdot \delta_{\ell_n} 
        + u^*_n + \delta_{u^*_n}
        &= 
        w \cdot \ell_n + u^*_{n+1}
    \\
        w \cdot \ell_{n-1} + u^*_n 
        &\leq
        w \cdot \ell_n + u^*_{n+1},
    \end{align*}
    where the last inequality holds by the assumption.
    Considering that $\alpha$ is monotone, we have that
    \begin{align*}
        \alpha(w \cdot \ell_{n-1} + u^*_n)
        &\leq
        \alpha(w \cdot \ell_n + u^*_{n+1})
    \\
        \ell_{n} 
        &\leq
        \ell_{n+1}.
    \end{align*}
    It follows that $\delta_{\ell_{n+1}} = \ell_{n+1} - \ell_n \geq 0$. 
    Furthermore, since $u^*_{n+1} \geq 0$ it follows that 
    $w \cdot \delta_{\ell_{n+1}} + \delta_{u^*_{n+1}} \geq 0$  holds. 
    The claim is proved.
    Then the fact that $(\ell_{n})_{n \geq n_0}$ is non-decreasing is the 
    direct consequence of the fact that the inequality 
    $w \cdot \delta_{\ell_{n}} + \delta_{u^*_{n}} \geq 0$  holds for 
    every $n \geq n_0$, thus $(\ell_{n})_{n \geq n_0}$ is monotone.
  The lemma is proved.
\end{proof}

We now prove Lemma~\ref{lemma:groupfree}.

\lemmagroupfree*
\begin{proof}
  Consider a semiautomaton $A = \langle \Sigma, Q, \delta \rangle$ that is not
  group-free.
  Consider cascade dynamics $D = \langle U, X, f \rangle$ of neurons are
  described in the statement of the lemma. 
  Let $\beta_j$ be the input function of the $j$-th neuron of $D$.
  Let $X = X_1 \times \dots \times X_d$.
  Let us assume by contradiction that $A$ is homomorphically represented by $D$
  with homomorphism $\psi$.
  Then, by Lemma~\ref{lemma:permutation_alternating_states}, 
  then there exist $x_0 \in X$ and $u \in U$ such that,
  for $x_i = f(x_{i-1},u)$ for $i \geq 1$, 
  the disequality
  $\psi(x_i) \neq \psi(x_{i+1})$ holds for every $i \geq 0$.
  In particular, the sequence $(\psi(x_i))_{i \geq 0}$ is not convergent.

  We show that $\psi$ is not continuous by showing that the sequence 
  $(x_i)_{i \geq 0}$ is convergent. This is a contradiction that proves the
  lemma.
  Note that state $x_i$ is of the form $\langle x_i^1, \dots, x_i^d \rangle$.
  By Lemma~\ref{lemma:convergence_of_iterated_sign} and
  Lemma~\ref{lemma:convergence_of_iterated_activation}, we have that
  the sequence $(x_i^1)_{i \geq 0}$ is convergent since
  $\beta_1(u)$ is constant and hence convergent.
  By induction,
  $(x_i^1)_{i \geq 0}, \dots, (x_i^{j-1})_{i \geq 0}$ are convergent,
  and hence the sequence 
  $$\big(\beta_j(u, x_i^1, \dots, x_i^{j-1})\big)_{i \geq 0}$$ is
  convergent since $\beta_j$ is continuous, and $u$ is constant.
  Thus, again by Lemma~\ref{lemma:convergence_of_iterated_sign} and
  Lemma~\ref{lemma:convergence_of_iterated_activation},
  we have that the sequence $(x_i^j)_{i \geq 0}$ is convergent.
  The lemma is proved.
\end{proof}

%

\subsubsection{Proof of Theorem~\ref{theorem:main_1_2}}

\theoremmainonetwo*
\begin{proof}
  The canonical automaton $A$ of $F$ is not group-free.
  By Lemma~\ref{lemma:groupfree}, automaton $A$ is not homomorphically
  represented by any RNC of sign or tanh neurons with positive weight.
  By Proposition~\ref{prop:homomorphism_converse}, there is no RNC of
  sign or tanh neurons with positive weight that is equivalent to $A$, and hence
  that implements $F$.
\end{proof}

\subsubsection{Proof of Theorem~\ref{theorem:expressivity_of_rnc_with_sign_tanh_positive_weight}}

\expressivityofrncwithsignandtanhwithposweights*
\begin{proof}
By Theorem~\ref{theorem:expressivity_of_rnc}, 
every group-free regular function can be implemented by RNCs of sign or tanh 
neurons with positive weight. By
Theorem~\ref{theorem:main_1_2}, no regular function that is not group-free
can be implemented by RNCs of sign or tanh 
neurons with positive weight. 
\end{proof}

\section{Proofs for Section `Necessary Conditions for Group-freeness'}

\subsection{Proof of Theorem~\ref{theorem:expressivity_of_rnn}}

\theoremexpressivityofrnn*
\begin{proof}
  As the proof of Theorem~\ref{theorem:expressivity_of_rnc}, except that
  cascades are replaced by networks, and Krohn-Rhodes
  Theorem~\ref{theorem:krohn_rhodes} is replaced by the
  Theorem~\ref{theorem:letichevsky}.
\end{proof}

\subsection{Proof of Theorem~\ref{theorem:expressivity_toggle}}

The proof of Theorem~\ref{theorem:expressivity_toggle} requires some 
preliminary definitions and propositions. We state and prove them below. 
Then we prove Theorem~\ref{theorem:expressivity_toggle}.

In the main body, we have considered neurons with positive weights.
In particular, we established conditions under which sign and tanh neurons can
implement a read functionality, along with set and reset functionalities.
Here we show that, when neurons have a negative weight, the inputs that before
induced a read functionality now induce a toggle functionality. 

We introduce a family of semiautomata, and corresponding neurons, that implement
set, reset, and a non-identity permutation transformation that we call `toggle'.
\begin{definition} \label{def:toggle_semiautomaton}
  Let $\mathit{toggle}$ be a distinguished symbol.
  The \emph{core toggle semiautomaton} is the semiautomaton 
  $\langle \Pi, Q, \delta \rangle$ where the inputs are
  $\Pi = \{ \mathit{set}, \mathit{reset}, \mathit{toggle} \}$,
  the states are
  $Q = \{ \mathit{high}, \mathit{low} \}$,
  and the following identities hold:
  \begin{align*}
      \delta(\mathit{toggle},\mathit{low}) = \mathit{high},
      \quad
      \delta(\mathit{toggle},\mathit{high}) = \mathit{low},
      \quad
      \delta(\mathit{set},q) = \mathit{high},
      \quad
      \delta(\mathit{reset},q) = \mathit{low}.
  \end{align*}
  A \emph{toggle semiautomaton} is the composition of an input function with the core
  toggle semiautomaton.
\end{definition}

\begin{definition}
  A \emph{core toggle neuron} is a core neuron $\langle V, X, f \rangle$ 
  where the set $V$ of inputs is expressed as the union of three disjoint 
  closed
  intervals $\vset,\vreset,\vtoggle$ of non-zero length, the set $X$ of states
  is expressed as the union of two disjoint closed intervals $\xlow,\xhigh$, and
  the following conditions hold:
  \begin{align*}
      f(\xlow,\vtoggle) \subseteq \xhigh,
      \;
      f(\xhigh,\vtoggle) \subseteq \xlow,
      \;
      f(X,\vset) \subseteq \xhigh,
      \;
      f(X,\vreset) \subseteq \xlow.
  \end{align*}
  A \emph{toggle neuron} is the composition of a core toggle neuron with an
  input function.
  The \emph{state interpretation} of a toggle neuron
  is the function $\psi$ defined as 
  $\psi(x) = \mathit{high}$ for $x \in \xhigh$ and
  $\psi(x) = \mathit{low}$ for $x \in \xlow$.
\end{definition}

The relationship between toggle neurons and toggle semiautomata can be
established in the same way it has been established in the case of flip-flops.
\begin{lemma}
  \label{lemma:toggle-homomorphism}
  Every toggle semiautomaton is homomorphically represented by a toggle neuron
  with arbitrary core.
  The homomorphism is given by the state interpretation of the neuron.
\end{lemma}
\begin{proof}
  Let us consider a toggle semiautomaton 
  $A = \langle \Sigma, Q, \delta_\phi \rangle$, where $\phi : \Sigma \to \Pi$ is
  its input function, where $\Pi = \{
    \mathit{toggle},\mathit{set},\mathit{reset}
  \}$.
  Let $N = \langle V, X, f \rangle$ be any core toggle neuron.
  Let $\epsilon$ be the minimum radius the intervals $\vtoggle,\vset,\vreset$ of
  $N$.
  Note that $\epsilon > 0$ by definition.
  Let
  $v_\mathrm{toggle}$,
  $v_\mathrm{set}$, and 
  $v_\mathrm{reset}$
  be the midpoint of 
  $\vtoggle,\vset,\vreset$, respectively.
  Note that
  $v_\mathrm{toggle} \pm \epsilon \in \vtoggle$,
  $v_\mathrm{set} \pm \epsilon \in \vset$, and
  $v_\mathrm{reset} \pm \epsilon \in \vreset$.
  Let $\xi: \Pi \to \{ v_\mathrm{toggle},v_\mathrm{set},v_\mathrm{reset}, \}$ be
  defined as 
  \begin{align*}
    \xi(\pi) = 
    \begin{cases}
      v_\mathrm{reset} & \text{ if } \pi = \mathit{reset}
      \\
      v_\mathrm{toggle} & \text{ if } \pi = \mathit{toggle}
      \\
      v_\mathrm{set} & \text{ if } \pi = \mathit{set}
    \end{cases}
  \end{align*}
  Consider $\beta = \lambda_\Sigma \circ \phi \circ \xi$.
  Note that $\beta$ is a continuous function since $\lambda_\Sigma$ is continuous and 
  $\phi,\xi$ are continuous because they are functions over finite discrete
  domains.
  Since $\beta$ is continuous and $U$ is compact by assumption,
  by the universal approximation theorem for feedforward neural networks (cf.\
  Theorem~2 of
  \cite{hornik1991approximation}), there exists an $\epsilon$-approximation
  $\beta'$ of $\beta$.
  Note that $\phi(\lambda_\Sigma(u)) = \mathit{set}$ implies $\beta'(u) \in \vset$,
  $\phi(\lambda_\Sigma(u)) = \mathit{reset}$ implies $\beta'(u) \in \vreset$,
  $\phi(\lambda_\Sigma(u)) = \mathit{toggle}$ implies $\beta'(u) \in \vtoggle$.

  Consider the toggle neuron $D = \langle U, X, f_{\beta'} \rangle$ obtained
  by composing $\beta'$ and $N$.

  First, we have that its state interpretation 
  $\psi : X \to Q$ with $Q = \{ \mathit{high}, \mathit{low} \}$ is a continuous
  function.
  To show this, it suffices to show that, for every $x_0 \in X$, and for every
  positive real number $\epsilon > 0$, there exists a positive real number 
  $\delta > 0$ such that for all $x \in X$ satisfying $d_X(x,x_0) < \delta$ also
  satisfy $d_Q(\psi(x),\psi(x_0)) < \epsilon$.
  The state space $X$ consists of two disjoint closed intervals 
  $X_\mathrm{high}$ and $X_\mathrm{low}$.
  Let $\delta$ be the minimum distance between elements in the two intervals,
  i.e.,
  \begin{align*}
    \delta = \min\left\{ d_X(y,z) \mid y \in X_\mathrm{high}\text{, } z \in
    X_\mathrm{low}\right\}.
  \end{align*}
  Such a minimum distance exists because 
  $X_\mathrm{high}$ and $X_\mathrm{low}$ are
  disjoint closed intervals.
  Now assume $d_X(x,x_0) < \delta$. It follows that 
  $\{ x,x_0 \} \subseteq X_\mathrm{high}$ or 
  $\{ x,x_0 \} \subseteq X_\mathrm{low}$.
  In both cases, $\psi(x) = \psi(x_0)$ and hence 
  $d_Q(\psi(x),\psi(x_0)) = 0 < \epsilon$ for every $\epsilon > 0$, as required.

  Second, for
  arbitrary $x \in X$ and $u \in U$, we show
  \begin{align*}
    \psi(f(x,u)) = \delta_\phi(\psi(x),\lambda_\Sigma(u)).
  \end{align*}
  Three cases are possible according to the value of $\phi(\lambda_\Sigma(u))$.
  \begin{itemize}
    \item 
      In the first case $\phi(\lambda_\Sigma(u)) = \mathit{reset}$. 
      We have that
      \begin{align*}
        \delta_\phi(\psi(x),\lambda_\Sigma(u))
        = \delta(\psi(x), \phi(\lambda_\Sigma(u))) 
        = \delta(\psi(x), \mathit{reset})= \mathit{low}.
      \end{align*}
      Thus, it suffices to show $f(x,u) \in \xlow$.
      Since $\phi(\lambda_\Sigma(u)) = \mathit{reset}$,
      we have that
      $\beta'(u) \in \vreset$ as noted above, and hence the required condition
      holds by the definition of toggle neuron. 

    \item
      In the second case $\phi(\lambda_\Sigma(u)) = \mathit{toggle}$. 
      We have that 
      \begin{align*}
        \delta_\phi(\psi(x),\lambda_\Sigma(u)) 
        = \delta(\psi(x), \phi(\lambda_\Sigma(u)))
        = \delta(\psi(x), \mathit{toggle}),
      \end{align*}
      which is $\mathit{high}$ if $\psi(x) = \mathit{low}$ and
      $\mathit{low}$ if $\psi(x) = \mathit{high}$.
      Thus, it suffices to show the two implications
      \begin{align*}
        x \in \xlow \quad & \Rightarrow \quad f(x,\lambda_\Sigma(u)) \in \xhigh,
      \\
        x \in \xhigh \quad & \Rightarrow \quad f(x,\lambda_\Sigma(u)) \in \xlow.
      \end{align*}
      Since $\phi(\lambda_\Sigma(u)) = \mathit{toggle}$,
      we have that $\beta'(u) \in \vtoggle$ as noted above, and hence the
      required implications hold by the definition of toggle neuron.

    \item
      In the third case $\phi(\lambda_\Sigma(u)) = \mathit{set}$. 
      We have that 
      \begin{align*}
        \delta_\phi(\psi(x), \lambda_\Sigma(u))
        = \delta(\psi(x), \phi(\lambda_\Sigma(u)))
        = \delta(\psi(x), \mathit{set}) = \mathit{high}.
      \end{align*}
      Thus, it suffices to show $f(x,\lambda_\Sigma(u)) \in \xhigh$.
      Since $\phi(\lambda_\Sigma(u)) = \mathit{set}$,
      we have that $\beta'(u) \in \vset$ as noted above, and hence the required
      condition holds by the definition of toggle neuron. 
  \end{itemize}
  The lemma is proved.
\end{proof}

Toggle neurons can be instantiated with neurons having sign or tanh activation,
and negative weight. The logic is similar to the one for flip-flop neurons,
except that now read inputs are replaced by toggle inputs, that determine a
change of sign in the state.

\begin{proposition}
\label{lemma:sign_with_negative_weight_transitions}
  A core neuron with sign activation and weight $w < 0$ 
  is a core toggle neuron if its state partition is 
  $\xlow = \{-1\}$ and $\xhigh = \{+1\}$ and
  its input partition satisfies: 
  \begin{align*}
    \vreset \leq w \cdot (a + 1), 
    \qquad 
    w \cdot (1 - a) \leq \vtoggle \leq w \cdot (a - 1),
    \qquad
    \vset \geq w \cdot (-a - 1),
  \end{align*}
  for some real number $a \in (0,1)$.
\end{proposition}
\begin{proof}
    We show that a neuron with states 
    $X = \xlow \cup \xhigh$, inputs $V = \vset \cup \vreset \cup \vread$, 
    weight $w < 0$ and sign activation function satisfies the conditions
    of the definition of a core toggle neuron. 
    First we show that conditions 
    \begin{align*}
      f(\xlow,\vtoggle) \subseteq \xhigh,
      \;
      f(\xhigh,\vtoggle) \subseteq \xlow,
      \;
      f(X,\vset) \subseteq \xhigh,
      \;
      f(X,\vreset) \subseteq \xlow
    \end{align*}
    are satisfied by
    \begin{align*}
        f(x,v) = \sign(w \cdot x + v),
    \end{align*}
    with $w < 0$, $x \in X$, and $v \in V$.

    Condition $f(\xlow,\vtoggle) \subseteq \xhigh$ is satisfied.
    By the premise of the proposition for $x \in \xlow$ and $v \in \vtoggle$ 
    it holds that 
    $x = -1$ and $v \geq w \cdot (1 - a)$. 
    Furthermore, $w \cdot (-a) > 0$ since $w < 0$ and $a > 0$.
    Thus,
    \begin{align*}
        w \cdot -1 + v 
        = |w| + v 
        \geq |w| + w \cdot (1 - a)
        = |w| + w + w \cdot (-a)
        = w \cdot (-a)
        > 0
    \end{align*}
    and hence
    \begin{align*}
        \sign(w \cdot -1 + v) = +1 \in \xhigh.
    \end{align*}

    Condition $f(\xhigh,\vtoggle) \subseteq \xlow$ is satisfied.
    By the premise of the proposition for $x \in \xhigh$ and $v \in \vtoggle$
    it holds that 
    $x = +1$ and $v \leq w \cdot (a - 1)$. 
    Furthermore, $w \cdot a < 0$ since $w < 0$ and $a > 0$. 
    Thus,
    \begin{align*}
        w \cdot +1 + v 
        = w + v 
        \leq w + w \cdot (a - 1)
        = w + w \cdot a + (-w)
        = w \cdot a
        < 0
    \end{align*}
    and hence
    \begin{align*}
        \sign(w \cdot +1 + v) = -1 \in \xlow.
    \end{align*}

    Condition $f(X,\vset) \subseteq \xhigh$ is satisfied.
    By the premise of the proposition for $x \in X$ and $v \in \vset$
    it holds that 
    $x \geq -1$ and $v \geq w \cdot (-a - 1)$.
    Furthermore, $w \cdot (-a) > 0$ since $w < 0$ and $a > 0$.
    Thus, 
    \begin{align*}
        w \cdot x + v 
        \geq w + v 
        \geq w + w \cdot (-a - 1)
        = w + w \cdot (-a) + (-w)
        = w \cdot (-a) 
        > 0
    \end{align*}
    and hence
    \begin{align*}
        \sign(w \cdot x + v) = +1 \in \xhigh.
    \end{align*}

    Condition $f(X,\vreset) \subseteq \xlow$ is satisfied.
    By the premise of the proposition for $x \in X$ and $v \in \vreset$
    it holds that 
    $x \leq +1$ and $v \leq w \cdot (a + 1)$.
    Furthermore, $w \cdot a < 0$ since $w < 0$ and $a > 0$.
    Thus, 
    \begin{align*}
        w \cdot x + v 
        \leq |w| + v  
        \leq |w| + w \cdot (a + 1)
        = |w| + w \cdot a + w
        = w \cdot a
        < 0
    \end{align*}
    and hence
    \begin{align*}
        \sign(w \cdot x + v) = -1 \in \xlow.
    \end{align*}

    It is left to show that the set $X$ of states and the set $V$ of inputs 
    satisfy the conditions of the definition of a core toggle neuron.
    The argument is symmetric to the one given in the proof of the 
    Proposition~\ref{prop:sign_transitions}.
\end{proof}

\begin{proposition}
\label{prop:tanh_with_negative_weight}
    Let $w < -1$, and let $f(x) = \tanh(w \cdot x)$.
    A core neuron with tanh activation and weight $w$ is a core toggle neuron if 
    its state partition is
    \begin{align*}
        \xlow = [-1, f(b)],
        \qquad
        \xhigh = [f(a), +1],
     \end{align*}
     for some $a < b$ satisfying $a + f(a) > b + f(b)$, 
     and its input partition satisfies
      \begin{align*}
      \vreset \leq w \cdot (b + 1),
      \qquad
      w \cdot (a - f(b))
      \leq \vtoggle \leq 
      w \cdot (b - f(a)),
      \qquad
      \vset \geq w \cdot (a - 1).
      \end{align*}
\end{proposition}
\begin{proof}
    We show that a neuron with states $X = \xlow \cup \xhigh$, inputs
  $V = \vset \cup \vreset \cup \vread$, weight $w < -1$ and tanh
  activation function satisfies the conditions of the definition of a core
  toggle neuron. First, we show that conditions
  \begin{align*}
      f(X,\vset) \subseteq \xhigh,
      \;
      f(X,\vreset) \subseteq \xlow,
      \;
      f(\xhigh,\vread) \subseteq \xhigh,
      \;
      f(\xlow,\vread) \subseteq \xlow,
  \end{align*}
  are satisfied by 
  \begin{align*}
    f(x,v) = \tanh(w \cdot x + v),
  \end{align*}
  with weight $w < -1$, $x \in X$ and $v \in V$.

    Condition $f(\xlow,\vtoggle) \subseteq \xhigh$ is satisfied.
    By the premise of the proposition for 
    $x \in \xlow$ and $v \in \vtoggle$ 
    it holds that 
    $x \leq f(b)$ and $v \geq w \cdot (a - f(b))$. Thus,
    \begin{align*}
        w \cdot x + v 
        \geq w \cdot f(b) + v
        \geq w \cdot f(b) + w \cdot (a -f(b))
        = w \cdot f(b) + w \cdot a + w \cdot (-f(b))
        = w \cdot a
    \end{align*}
    Since tanh is monotonic it holds that
    $\tanh(w \cdot x + v) \geq \tanh(w \cdot a) = f(a)$, and hence
    \begin{align*}
        \tanh(w \cdot x + v) \in \xhigh.
    \end{align*}

    Condition $f(\xhigh,\vtoggle) \subseteq \xlow$ is satisfied.
    By the premise of the proposition for 
    $x \in \xhigh$ and $v \in \vtoggle$
    it holds that 
    $x \geq f(a)$ and $v \leq w \cdot (b - f(a))$. 
    Thus,
    \begin{align*}
        w \cdot x + v 
        \leq w \cdot f(a) + v 
        \leq w \cdot f(a) + w \cdot (b - f(a))
        = w \cdot f(a) + w \cdot b + w \cdot (-f(a))
        = w \cdot b
    \end{align*}
    Since tanh is monotonic it holds that
    $\tanh(w \cdot x + v) \leq \tanh(w \cdot b) = f(b)$, and hence
    \begin{align*}
        \tanh(w \cdot x + v) \in \xlow.
    \end{align*}

    Condition $f(X,\vset) \subseteq \xhigh$ is satisfied.
    By the premise of the proposition for 
    $x \in X$ and $v \in \vset$
    it holds that 
    $x \leq +1$ and $v \geq w \cdot (a - 1)$.
    Thus, 
    \begin{align*}
        w \cdot x + v 
        \geq w + v
        \geq w + w \cdot (a - 1)
        = w + w \cdot a + (-w)
        = w \cdot a
    \end{align*}
    Since tanh is monotonic it holds that
    $\tanh(w \cdot x + v) \geq \tanh(w \cdot a) = f(a)$, and hence
    \begin{align*}
        \tanh(w \cdot x + v) \in \xhigh.
    \end{align*}

    Condition $f(X,\vreset) \subseteq \xlow$ is satisfied.
    By the premise of the proposition for $x \in X$ and $v \in \vreset$
    it holds that 
    $x \geq -1$ and $v \leq w \cdot (b + 1)$.
    Thus, 
    \begin{align*}
        w \cdot x + v 
        \leq |w| + v
        \leq |w| + w \cdot (b + 1)
        = |w| + w \cdot b + w
        = w \cdot b.
    \end{align*}
    Since tanh is monotonic it holds that
    $\tanh(w \cdot x + v) \leq \tanh(w \cdot b) = f(b)$, and hence
    \begin{align*}
        \tanh(w \cdot x + v) \in \xlow.
    \end{align*}

  Finally, we show that the set $X$ of states and the set $V$ of inputs 
  satisfy the conditions of the definition of a core toggle neuron.

  The set $X$ is the union of two disjoint closed intervals. Intervals
  $\xlow$ and $\xhigh$ are bounded on the both sides, hence they are closed
  by definition.
  By the premise of the proposition $a < b$ and $w < -1$ and thus, 
  $w \cdot a > w \cdot b$. Since tanh is monotonic, we have that 
  $f(a) > f(b)$ and therefore, $\xlow$ and $\xhigh$ are disjoint.

  The set $V$ is the union of three disjoint closed intervals of non-zero 
  length. 
  By the premise of the proposition  
  the interval 
  $\vset$ is bounded only on the left by a number $w \cdot (a - 1)$, 
  the interval 
  $\vreset$ is bounded only on the right by a number $w \cdot (b + 1)$,
  and the interval $\vtoggle$ is bounded on both sides, hence
  all three intervals are closed by definition.
  Intervals $\vreset$ and $\vset$ are trivially of non-zero length, since
  $w \cdot (a - 1) < +\infty$ and $w \cdot (b + 1) > -\infty$.
  To show that interval $\vtoggle$ has non-zero length it suffices to show 
  that numbers $a < b$ such that $a + f(a) > b + f(b)$ exist. 
  Then the fact that $\vtoggle$ has non-zero length follows trivially by
  the following inequality
  \begin{align*}
    b + f(b) &< a + f(a),
\\
    -b -f(b) &> -a -f(a), 
\\
    a - f(b) &> b - f(a), 
\\
    w \cdot (a - f(b)) &< w \cdot (b - f(a)).
  \end{align*}
  Consider derivative $\tanh'(w \cdot x)$ for $w < -1$. 
  Its shape is an upside-down bell reaching minimum value of $w$ at $x = 0$, 
  with tails approaching zero for values of $x$ going towards $\pm\infty$
  Let $g(x)  = x + \tanh(w \cdot x)$. 
  It is derivative is $g'(x) = 1 + \tanh'(w \cdot x)$. 
  Then $g'(x) = 1 + w < 0$ for $x = 0$ since $w < -1$, and
  $g'(x)$ is approaching 1 for values of $x$ going towards $\pm\infty$ 
  because $\tanh'(w \cdot x)$ is approaching zero.
  Then, there exist two points $-p < 0 < p$ where $g'(x)$ crosses the 
  $x$-axis and thus $g'(x) < 0$ for $-p < x < p$. Thus, function
  $g(x)$ is decreasing between $-p$ and $p$. 
  Therefore, two numbers $a,b$ can be found such that 
  $a < b$ and $a + f(a) > b + f(b)$.

  The intervals $\vreset, \vtoggle, \vset$ are disjoint.
  By the property of tanh we have that $f(a), f(b) < 1$ and by the premise of 
  the proposition $a < b$. Then, the following inequalities hold
  \begin{align*}
      w \cdot b + w  &< w \cdot a - w \cdot f(b),
    \\
      w \cdot b - w \cdot f(a) &< w \cdot a - w.
  \end{align*}
  The proposition is proved.
\end{proof}
We highlight two differences between the tanh toggle and the tanh flip-flop.
First, the constants $a$ and $b$ occur swapped in the state boundaries.
Second, they occur mixed in the boundaries of the toggle input interval.
They are both consequences of $w$ negative, in particular of the fact it
introduces antimonotonicity in the dynamics function.
Toggle neurons such as the ones above allow for going beyond group-free regular
functions.
In particular, we can show a toggle neuron that recognises the language
consisting of all strings of odd length, which is regular but not group-free. 

We now prove Theorem~\ref{theorem:expressivity_toggle}.

\theoremexpressivitytoggle*
\begin{proof}
  Let us consider the automaton acceptor
  $A = \langle \Sigma, Q, \delta, q^\mathrm{init}, \Gamma, \theta \rangle$ 
  with singleton input alphabet $\Sigma = \{ a \}$,
  semiautomaton $D$ given by the composition of the toggle semiautomaton
  (Definition~\ref{def:toggle_semiautomaton}) with the input function 
  $\phi(a) = \mathit{toggle}$, initial state $q^\mathrm{init} = \mathit{low}$,
  output alphabet is $\Gamma = \{ 0,1 \}$,
  and output function is defined as 
  \begin{align*}
    \theta(\mathit{low}, a) & = 1,
    \\
    \theta(\mathit{high}, a) & = 0.
  \end{align*}
  Note that the automaton accepts the language of strings $aa \cdots a$
  having odd length.

  We argue that the semiautomaton $D$ of $A$ is not group-free. 
  It has only one transformation $\tau$, induced by $a$, which is
  $\tau(\mathit{low}) = \mathit{high}$ and $\tau(\mathit{high}) = \mathit{low}$.
  Note that $\tau$ is a permutation.
  It generates the transformation semigroup $G = \{ \tau, e \}$ where $e$ is
  identity, and it is generated as $\tau \circ \tau$---in fact, 
  $\tau(\tau(\mathit{low})) = \mathit{low}$ and $\tau(\tau(\mathit{high})) =
  \mathit{high}$.
  Thus, $G$ is the characteristic semigroup of $D$, and it is in particular a
  group. This shows that $D$ is not group-free. Formally, it is not group-free
  because $G$ is divided by $G$, which is a non-trivial group. 

  Then, by Lemma~\ref{lemma:toggle-homomorphism}, we have that $D$ is
  homomorphically represented by a toggle neuron with arbitrary core.
  By Propositions~\ref{lemma:sign_with_negative_weight_transitions}
  and~\ref{prop:tanh_with_negative_weight}, it follows that 
  $D$ is homomorphically represented by a neuron $N$ with sign or tanh
  activation (and negative weight).
  The rest of the proof uses the argument already used in
  Theorem~\ref{theorem:expressivity_of_rnc}.

  By Proposition~\ref{prop:homomorphism},
  there is an RNC $S$ with dynamics $N$ that implements the same function as 
  the system $A_{\lambda_\Sigma}$, obtained as the composition of $A$ with the
  fixed symbol grounding $\lambda_{\Sigma}$.
  The output function of $S$ can be $\epsilon$-approximated by a feedforward
  neural network $h$, by the universal approximation theorem, cf.\
  \cite{hornik1991approximation}.
  Thus, the sought RNC $S'$ is obtained from $S$ by replacing its output
  function with the feedforward neural network $h$.
  Since the output symbol grounding $\lambda_\Gamma$ is $\epsilon$-robust,
  the system $S'_{\lambda_\Gamma}$ obtained by the composition of $S'$ with the
  output symbol grounding $\lambda_\Gamma$ is equivalent to
  $A_{\lambda_\Sigma}$. This proves the theorem.
  %
  %
\end{proof}

\section{Proofs for Section `Implementation of Group Neurons'}

\subsection{Proof of Proposition~\ref{prop:ctwo_sign}}
We prove Proposition~\ref{prop:ctwo_sign}.

\propctwosign*

\begin{proof}
We show that a neuron described in the proposition satisfies the conditions
in the definition of a core $C_2$-neuron. First, we show that conditions
\begin{align*}
f(X_0, V_0) \subseteq X_0, \;
f(X_0, V_1) \subseteq X_1, \;
f(X_1, V_0) \subseteq X_1, \;
f(X_1, V_1) \subseteq X_0, \;
\end{align*}
are satisfied by
\begin{align*}
f(x,v) = \sign(w \cdot x \cdot v),
\end{align*}
for every $x \in X, v \in V,$ and $a, w > 0$ or $a, w < 0$. 
For all derivations below we note that $f(a) = +1$ for both
$a,w <0$ and $a, w > 0$. 

Condition $f(X_0, V_0) \subseteq X_0$ is satisfied. 
By the premise of the proposition, for $x \in X_0, v \in V_0$ it holds that
$x = -1$ 
and $v \geq a$ for $a, w > 0$ 
and $v \leq a$ for $a, w < 0$. 
In either case, we have that
\begin{align*}
    w \cdot x \cdot v 
    =
    w \cdot -1 \cdot v 
    < 0
\end{align*}
and hence
\begin{align*}
\sign(w \cdot x \cdot v) = -1 \in X_0.
\end{align*}
Condition $f(X_0, V_1) \subseteq X_1$ is satisfied. 
By the premise of the proposition, for $x \in X_0, v \in V_1$ it holds that 
$x  = -1$
and $v \leq -a$ for $a, w > 0$ 
and $v \geq -a$ for $a, w < 0$. 
In either case, we have that
\begin{align*}
    w \cdot x \cdot v 
    = 
    w \cdot -1 \cdot v 
    > 0
\end{align*}
and hence
\begin{align*}
\sign(w \cdot x \cdot v) = +1 \in X_1.
\end{align*}
Condition $f(X_1, V_0) \subseteq X_1$ is satisfied. 
By the premise of the proposition, for $x \in X_1, v \in V_0$ it holds that 
$x = +1$
and $v \geq a$ for $a, w > 0$ 
and $v \leq a$ for $a, w < 0$. 
In either case, we have that
\begin{align*}
    w \cdot x \cdot v 
    =
    w \cdot +1 \cdot v 
    > 0
\end{align*}
and hence
\begin{align*}
\sign(w \cdot x \cdot v) = +1 \in X_1.
\end{align*}
Condition $f(X_1, V_1) \subseteq X_0$ is satisfied. 
By the premise of the proposition, for $x \in X_1, v \in V_1$ it holds that 
$x = +1$
and $v \leq -a$ for $a, w > 0$ 
and $v \geq -a$ for $a, w < 0$. 
In either case, we have that
\begin{align*}
    w \cdot x \cdot v 
    =
    w \cdot +1 \cdot v 
    < 0
\end{align*}
and hence

\begin{align*}
    \sign(w \cdot x \cdot v) = -1 \in X_0.
\end{align*}
Finally, we show that the set $X$ of states and the set $V$ of inputs satisfy
the conditions in the definition of the core $C_2$-neuron.
The set $X$ is the union of two disjoint closed intervals. 
This condition is satisfied by definition, since
$X_0 = \{-1\}$ and $X_1 = \{+1\}$.

The set $V$ is the union of two disjoint closed intervals of non-zero length. 
By 
the premise of the proposition the interval $V_0$ is either bounded only on the 
right by a number $a <0$ or it is bounded only on the left by a number
$a > 0$. The interval $V_1$ is either bounded only on the 
right by a number $-a <0$ or it is bounded only on the left by a number 
$a > 0$. Hence, both of the intervals are closed and are
of non-zero length by definition. 
Furthermore, by the premise of the proposition the number $a$ is 
either greater than zero or it is less than zero, thus two intervals 
are are disjoint.
The proposition is proved.
\end{proof}

\subsection{Proof of Proposition~\ref{prop:ctwo_tanh}}

We prove Proposition~\ref{prop:ctwo_tanh}.

\propctwotanh*

\begin{proof}
We show that a neuron described in the proposition satisfies the conditions
in the definition of a core $C_2$-neuron. First, we show that conditions
\begin{align*}
f(X_0, V_0) \subseteq X_0, \;
f(X_0, V_1) \subseteq X_1, \;
f(X_1, V_0) \subseteq X_1, \;
f(X_1, V_1) \subseteq X_0, \;
\end{align*}
are satisfied by
\begin{align*}
f(x,v) = \tanh(w \cdot x \cdot v),
\end{align*}
for every $x \in X, v \in V,$ and $a, w > 0$ or $a, w < 0$. 
For all derivations below we note that $f(a) > 0$ for both
$a,w <0$ and $a, w > 0$. 

Condition $f(X_0, V_0) \subseteq X_0$ is satisfied. 
By the premise of the proposition, for $x \in X_0, v \in V_0$ it holds that
$x \leq -f(a)$ 
and $v \geq a/f(a)$ for $a, w > 0$ and $v \leq a/f(a)$ for 
$a, w < 0$. In either case, we have that
\begin{align*}
    w \cdot x \cdot v 
    \leq 
    w \cdot -f(a) \cdot v
    \leq 
    w \cdot -f(a) \cdot \frac{a}{f(a)} 
    = 
    w \cdot -a.
\end{align*}
Since $\tanh$ is monotonic it holds that 
$\tanh(w \cdot x \cdot v) \leq \tanh(w \cdot -a) = -f(a)$,
and hence
\begin{align*}
\tanh(w \cdot x \cdot v) \in X_0.
\end{align*}
Condition $f(X_0, V_1) \subseteq X_1$ is satisfied. 
By the premise of the proposition, for $x \in X_0, v \in V_1$ it holds that 
$x \leq -f(a)$
and $v \leq -a/f(a)$ for $a, w > 0$ 
and $v \geq -a/f(a)$ for $a, w < 0$. 
In either case, we have that
\begin{align*}
w \cdot x \cdot v 
\geq
w \cdot -f(a) \cdot v
\geq
w \cdot -f(a) \cdot -\frac{a}{f(a)}
=
w \cdot a.
\end{align*}
Since $\tanh$ is monotonic it holds that 
$\tanh(w \cdot x \cdot v) \geq \tanh(w \cdot a) = f(a)$, and hence
\begin{align*}
\tanh(w \cdot x \cdot v) \in X_1.
\end{align*}
Condition $f(X_1, V_0) \subseteq X_1$ is satisfied. 
By the premise of the proposition, for $x \in X_1, v \in V_0$ it holds that 
$x \geq f(a)$
and $v \geq a/f(a)$ for $a, w > 0$ 
and $v \leq a/f(a)$ for $a, w < 0$. 
In either case, we have that
\begin{align*}
w \cdot x \cdot v 
\geq
w \cdot f(a) \cdot v
\geq
w \cdot f(a) \cdot \frac{a}{f(a)}
=
w \cdot a.
\end{align*}
Since $\tanh$ is monotonic it holds that 
$\tanh(w \cdot x \cdot v) \geq \tanh(w \cdot a) = f(a)$, and hence
\begin{align*}
\tanh(w \cdot x \cdot v) \in X_1.
\end{align*}
Condition $f(X_1, V_1) \subseteq X_0$ is satisfied. 
By the premise of the proposition, for $x \in X_1, v \in V_1$ it holds that 
$x \geq f(a)$
and $v \leq -a/f(a)$ for $a, w > 0$ 
and $v \geq -a/f(a)$ for $a, w < 0$. 
In either case, we have that
\begin{align*}
w \cdot x \cdot v 
\leq
w \cdot f(a) \cdot v
\leq
w \cdot f(a) \cdot -\frac{a}{f(a)}
=
w \cdot -a.
\end{align*}
Since $\tanh$ is monotonic it holds that 
$\tanh(w \cdot x \cdot v) \leq \tanh(w \cdot -a) = -f(a)$, and hence
\begin{align*}
\tanh(w \cdot x \cdot v) \in X_0.
\end{align*}
Finally, we show that the set $X$ of states and the set $V$ of inputs satisfy
the conditions of the definition of a core $C_2$-neuron. The set $X$ is the
union of two disjoint closed intervals. 
Intervals $X_0$ and $X_1$ are closed by definition, since their bounded by 
a number on both sides.
Furthermore, for $a,w > 0$ and $a,w < 0$ it holds that
$\tanh(w \cdot a) > 0$. 
Thus, $-f(a)=-\tanh(w \cdot a) < \tanh(w \cdot a) = f(a)$, hence the 
intervals are disjoint.

The set $V$ is the union of disjoint closed intervals of non-zero length. 
By 
the premise of the proposition the interval $V_0$ is either bounded only on 
the left by a number $a/f(a)$ with $a > 0$ or it is bounded only on the right
by a number $a/f(a)$ with $a < 0$. 
The interval $V_1$ is either bounded only on the right by a number 
$-a/f(a)$ with $a > 0$ or it is bounded only on the left by a number
$-a/f(a)$ with $a < 0$. Thus, by definition both intervals are closed and 
have non-zero length.
Furthermore, since $f(a) > 0$ for both $a,w > 0$ and $a,w < 0$ it holds
that $V_0$ and $V_1$ are disjoint.
The proposition is proved.
\end{proof}

%% file: sections/example-details.tex

\section{Examples}
In this section we provide two examples of a star-free (group-free) regular
language and a group-free regular function from two different applications:
Stock market analysis and Reinforcement learning.
We also construct cascades of flip-flops capturing them, which in particular
proves they are group-free.

\subsection{Example from Stock Market Analysis}

\begin{example}
  \label{example:stocks}
  \emph{Charting patterns} are a commonly-used tool in the analysis of 
  financial
  data \citep{jpmorgan}.
  The charting technique may be operationalized through \emph{trading rules} of
  the form \citep{leigh2002stock}:
    ``If charting pattern $X$ is identified in the previous $N$ trading days,
    then buy; and sell on the $Y$th trading day after that. If charting pattern
    $X$ is identified in the previous $N$ trading days, then sell.''
  Many charting patterns correspond to \emph{star-free regular languages}.
  An example is the \emph{Broadening Tops} (BTOP) pattern that is
  characterised by a sequence of five consecutive local extrema $E_1, \dots,
  E_5$ such that
  $E_1$ is a maximum, $E_1 < E_3 < E_5$, and $E_2 > E_4$ (see
  Definition~2 of \citep{lo2000foundations}).
\end{example}

\begin{figure}[h]
    \centering
    \includegraphics[width=0.8\textwidth]{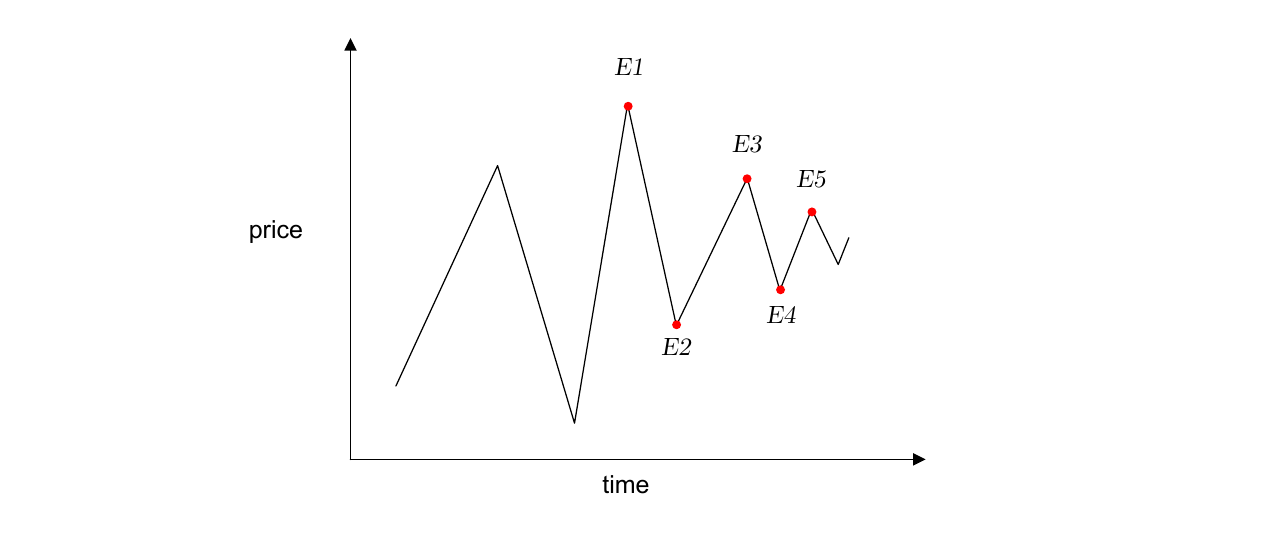}
    \caption{A price chart matching the TTOP pattern.}
    \label{fig:stocks-ttop}
\end{figure}

The TTOP pattern is shown in Figure~\ref{fig:stocks-ttop}. 
We next show that there is cascade of flip-flops that recognises TTOP,
thus this pattern is star-free,  
by the converse of Theorem~\ref{theorem:krohn_rhodes}.

We first provide the pseudocode of an algorithm for detecting TTOP.
Then we show how this algorithm can be implemented by a cascade of flip-flops.
The algorithm is higher-level than the cascade; at the same time it has been
devised to ensure an easy translation to cascade.

\subsubsection{Algorithm to Detect TTOP}

\begin{algorithm}[h]
  \SetKw{Or}{or}%
  \SetKw{And}{and}%
  \SetKw{Yield}{yield}%
  \SetKwFor{Loop}{loop}{}{}%
  \SetKwFor{Until}{until}{do}{}%
  \SetKwInput{Input}{Input}
  \SetKwInput{OInput}{Optional}
  \SetKwInput{Output}{Output}
  \SetKwInput{Parameters}{Parameters}
  \SetKwProg{Fn}{function}{}{}

  \BlankLine
  $\mathit{lastMax} \leftarrow \mathit{maxPossibleStockPrice}+1$ \tcp*{greater
  than max stock price}
  $\mathit{lastMin} \leftarrow \mathit{minPossibleStockPrice}-1$ \tcp*{smaller
  than min stock price}
  $\mathit{count}_0 \leftarrow \mathit{true}$\;
  $\mathit{count}_1,\mathit{count}_2,\mathit{count}_3,\mathit{count}_4,\mathit{count}_5 \leftarrow \mathit{false}$\;
  $\mathit{slope} \leftarrow \mathit{positive}$\;
  $\mathit{prev} \leftarrow$ read the next element from the input sequence of
  stock prices\;

  \BlankLine
  \Loop{}{
    $\mathit{cur} \leftarrow$ read the next element from the input sequence of
  stock prices\;

  \BlankLine

    \tcp{check if prev is an extremum and if so add it} 

    \If{$\mathit{slope} = \mathit{positive}$ \emph{\textbf{and}} $\mathit{prev}
      > \mathit{cur}$}{
      \tcp{prev is a local maximum}
      \If{$\mathit{prev} < \mathit{lastMax}$}{
        \tcp{found new valid extremum}
        $\mathit{lastMax} \leftarrow \mathit{prev}$\;
        \textit{IncrementCount()}\;
      }
      \Else{
        \tcp{the new maximum violates the previous one}
        \tcp{the new maximum is now a valid initial extremum}
        \textit{SetCountToOne()}\;
      }
    }
    \If{$\mathit{slope} = \mathit{negative}$ \emph{\textbf{and}} $\mathit{prev}
      < \mathit{cur}$}{
      \tcp{prev is a local minimum}
      \If{$\mathit{prev} > \mathit{lastMin}$}{
        \tcp{found new valid extremum}
        $\mathit{lastMin} \leftarrow \mathit{prev}$\;
        \textit{IncrementCount()}\;
      }
      \Else{
        \tcp{the new minimum violates the previous one}
        \tcp{the last maximum is now a valid initial extremum}
        \tcp{and the new minimum is now a valid second extremum}
        \textit{SetCountToTwo()}\;
      }
    }

    \BlankLine

    \tcp{update slope} 
    \If{$\mathit{prev} < \mathit{cur}$}{
      $\mathit{slope} \leftarrow \mathit{positive}$\;
    }
    \ElseIf{$\mathit{prev} > \mathit{cur}$}{
      $\mathit{slope} \leftarrow \mathit{negative}$\;
    }

    \BlankLine

    \tcp{output whether the sequence so far ends with a TTOP}
    \If{$\mathit{count}_5$}{
      output $1$\;
    }
    \Else{
      output $0$\;
    }

    \BlankLine

    $\mathit{prev} \leftarrow \mathit{cur}$\;
  }
  \caption{Detect TTOP}
  \label{algorithm:ttop}
\end{algorithm}

\begin{algorithm}[t]
  \SetKw{Or}{or}%
  \SetKw{And}{and}%
  \SetKw{Yield}{yield}%
  \SetKwFor{Loop}{loop}{}{}%
  \SetKwFor{Until}{until}{do}{}%
  \SetKwInput{Input}{Input}
  \SetKwInput{OInput}{Optional}
  \SetKwInput{Output}{Output}
  \SetKwInput{Parameters}{Parameters}
  \SetKwProg{Fn}{Function}{ :}{}

      \BlankLine

  \Fn{IncrementCount()}{
  \If{$\mathit{count}_0$}{
    $\mathit{count}_0 \leftarrow \mathit{false}$\;
    $\mathit{count}_1 \leftarrow \mathit{true}$\;
        }
        \If{$\mathit{count}_1$}{
          $\mathit{count}_1 \leftarrow \mathit{false}$\;
          $\mathit{count}_2 \leftarrow \mathit{true}$\;
        }
        \If{$\mathit{count}_2$}{
          $\mathit{count}_2 \leftarrow \mathit{false}$\;
          $\mathit{count}_3 \leftarrow \mathit{true}$\;
        }
        \If{$\mathit{count}_3$}{
          $\mathit{count}_3 \leftarrow \mathit{false}$\;
          $\mathit{count}_4 \leftarrow \mathit{true}$\;
        }
        \If{$\mathit{count}_4$}{
          $\mathit{count}_4 \leftarrow \mathit{false}$\;
          $\mathit{count}_5 \leftarrow \mathit{true}$\;
        }
      }

      \BlankLine

  \Fn{SetCountToOne()}{
  $\mathit{count}_0,\mathit{count}_2,\mathit{count}_3,\mathit{count}_4,
  \mathit{count}_5 \leftarrow \mathit{false}$\;
  $\mathit{count}_1 \leftarrow \mathit{true}$\;
  }

      \BlankLine

  \Fn{SetCountToTwo()}{
  $\mathit{count}_0,\mathit{count}_1,\mathit{count}_3,\mathit{count}_4,
  \mathit{count}_5 \leftarrow \mathit{false}$\;
  $\mathit{count}_2 \leftarrow \mathit{true}$\;
  }
  \caption{Count Subroutines}
  \label{algorithm:count_subroutines}
\end{algorithm}

We describe the algorithm for detecting TTOP, whose the pseudocode is given in 
Algorithm~\ref{algorithm:ttop}---at the end of the appendix.

The algorithm looks for five consecutive local extrema $E_1,E_2,E_3,E_4,E_5$
satisfying the required conditions.
In particular, it is required that $E_1$ is a local maximum. It follows that
$E_2$ is a local minimum, $E_3$ is a local maximum, $E_4$ is a local minimum,
and $E_5$ is a local maximum.
The argument is that we cannot have two local extrema of the same kind in a row.
For instance, if a local maximum $E_1$ was followed by another local maximum
$E_2$, then it would contradict the fact that $E_1$ was a local maximum.
Please see Figure~\ref{fig:stocks-ttop}.

Next we describe Algorithm~\ref{algorithm:ttop}.

\paragraph{Variables.}
The variable $\mathit{lastMax}$ defined in Line~1 is for keeping track of the
last maximum stock price that can be considered as a maximum in a TTOP pattern.
Similarly, the variable $\mathit{lastMin}$ defined in Line~2 is for keeping
track of the last minimum stock price that can be considered as a minimum of a
TTOP pattern.
They are initialised with boundary values, so that they will be
replaced at the first comparison with an actual stock price.
Lines~3 and~4 define five Boolean variables 
$\mathit{count}_0,\mathit{count}_1,\mathit{count}_2,\mathit{count}_3,
\mathit{count}_4,\mathit{count}_5$.  
They represent a count. Initially, $\mathit{count}_0$ is true and the others are
false. This represents a count of zero.
The count stands for the number of valid consecutive extrema we have found. When
the count is $5$, it means that we have found $E_1,E_2,E_3,E_4,E_5$, and hence
the input sequence received so far matches the TTOP pattern. 
The variable $\mathit{slope}$ defined in Line~5 stores whether the slope of the
last stretch of the chart is positive or negative.
It is initialised to positive because this way the first stock price we
encounter before a price drop will be correctly considered a local maximum.

\paragraph{The main loop.}
At each iteration the algorithm will operate over the current stock price stored
in variable $\mathit{cur}$ and the previous stock price stored in a variable
called $\mathit{prev}$.
Before entering the main loop, Line~6 reads and stores the first stock price in
the input sequence. This ensures that $\mathit{prev}$ is available right from
the first iteration.
Every iteration of the main loop starts by reading the next stock price from the
input sequence (Line~8).

\paragraph{Encountering a new local maximum.}
Lines~9--14 describe the operations to perform upon detecting a new local
maximum.
In particular, we have that $\mathit{prev}$ is a new local maximum if current
slope is positive (up to $\mathit{prev}$) and the current value
$\mathit{cur}$ is smaller than $\mathit{prev}$ (Line~9).
Then we have two possibilities.
If $\mathit{prev}$ is smaller than the last maximum $\mathit{lastMax}$, then
$\mathit{prev}$ is a valid extremum to be included in the TTOP pattern detected
so far (Line~10). In this case, it becomes the new $\mathit{lastMax}$ (Line~11),
and we increment the count of extrema matched so far (Line~12).
Note that the IncrementCount subroutine---described below---increments up to
$5$.
In the other case, we have that 
$\mathit{prev}$ is bigger than the last maximum $\mathit{lastMax}$. Then 
$\mathit{prev}$ invalidates the sequence of consecutive extrema matched before
encountering $\mathit{prev}$. At the same time $\mathit{prev}$ can be the
initial maximum $E_1$ of a TTOP pattern to be matched from now on, so we set the
count to one (Line~14).

\paragraph{Encountering a new local minimum.}
Lines~15--20 describe the operations to perform upon detecting a new local
minimum. They are dual to the case of a maximum. The only difference is that,
when $\mathit{prev}$ invalidates the previous sequence of consecutive extrema
already matched, we have that $\mathit{lastMax}$ is still a valid initial
maximum $E_1$, and $\mathit{prev}$ is a valid minium $E_2$. Thus, we set the
count to two.
Note also that a local minimum can only be encountered after the first local
maximum.

\paragraph{Slope update.}
The variable $\mathit{slope}$ is initialised in Line~5 and it is updated at
every iteration in Lines~21--24.
Its initialisation and update ensure that, when the variable is read in Lines~9
and~15, its content specifies whether the slope of the chart right before
$\mathit{prev}$ is positive or negative.
If a plateau is encountered, the value of the slope is maintained.
This ensures that local extrema are detected correctly.

\paragraph{Count subroutines.}
The suboroutines to handle the count are described in  
Algorithm~\ref{algorithm:count_subroutines}---at the end of the appendix.
Regarding the subroutine to increment the count,
note that, when the count is $5$, no update happens, as it is enough for the
algorithm to know whether at least $5$ valid extrema have been matched.

\subsubsection{Cascade to Detect TTOP}

\begin{figure}[t]
    \centering
    \includegraphics[width=0.7\textwidth]{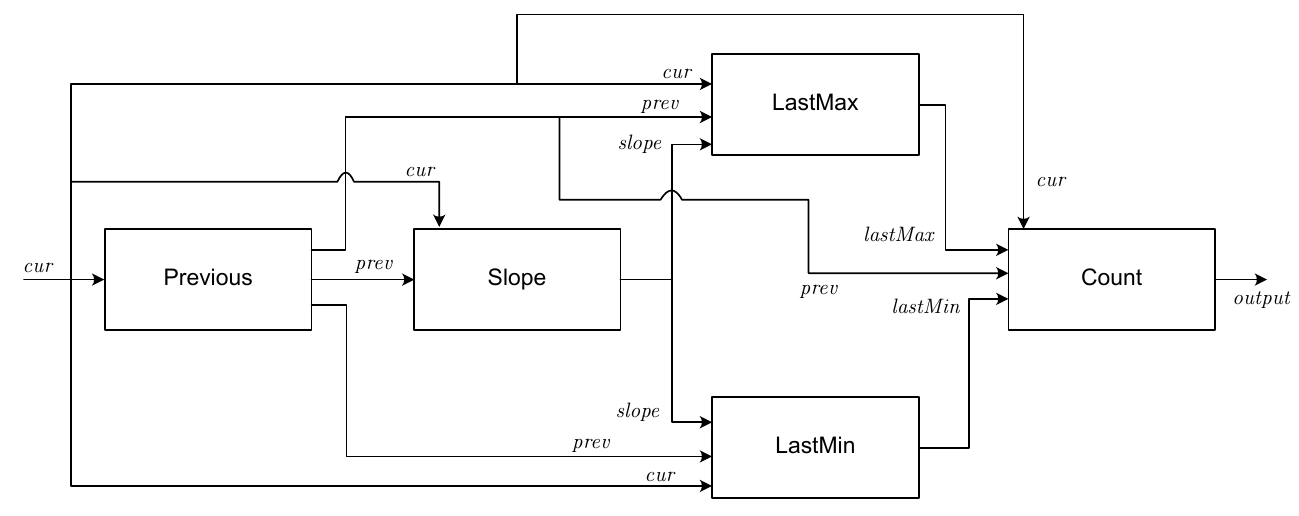}
    \caption{Diagram of the cascade for stock prediction.}
    \label{fig:stocks-overall-diagram}
\end{figure}

\begin{figure}[p!]
    \centering
    \includegraphics[width=0.7\textwidth]{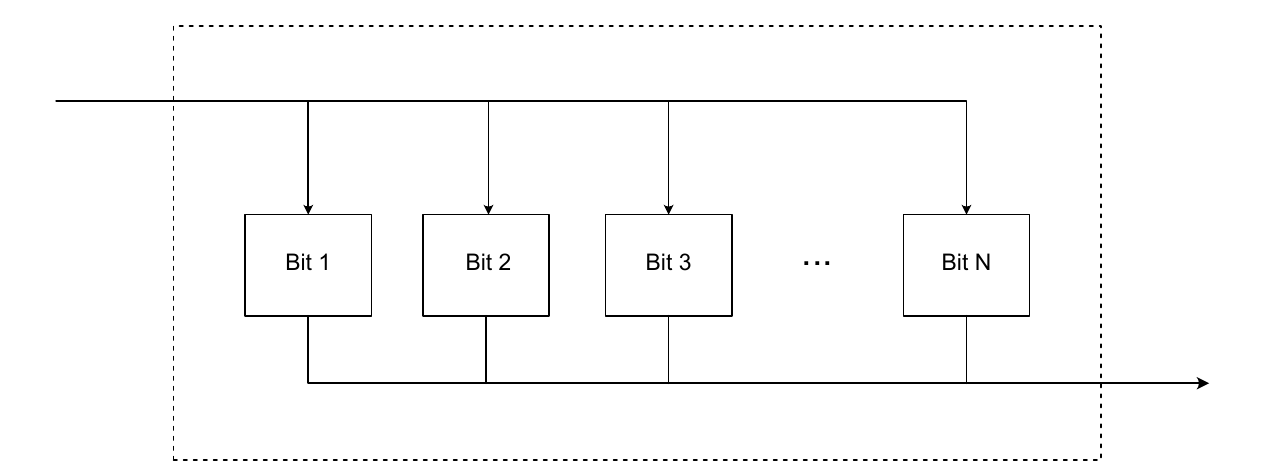}
    \caption{Diagram of Previous.}
    \label{fig:stocks-previous}
\end{figure}

\begin{figure}[p!]
    \centering
    \includegraphics[width=0.7\textwidth]{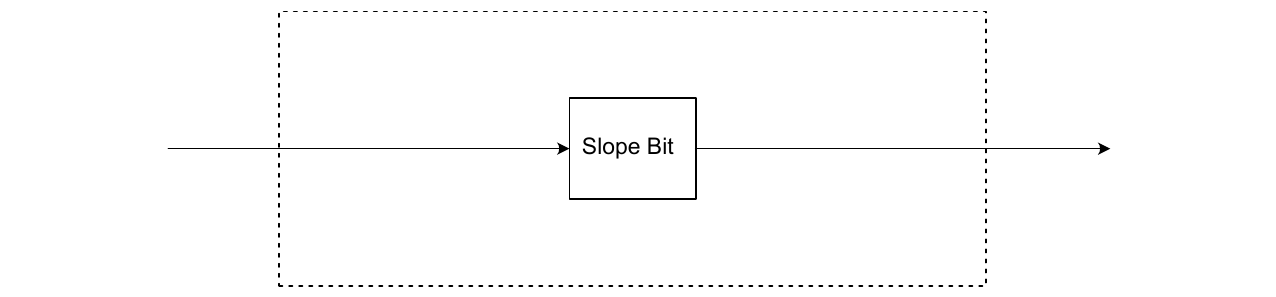}
    \caption{Diagram of Slope.}
    \label{fig:stocks-slope}
\end{figure}

\begin{figure}[p!]
    \centering
    \includegraphics[width=0.7\textwidth]{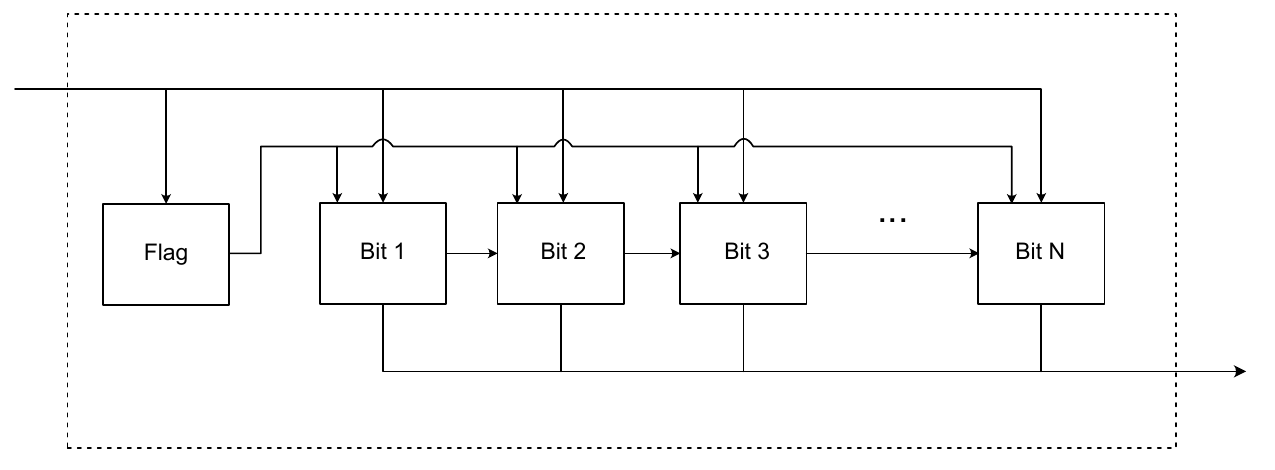}
    \caption{Diagram of LastMax and LastMin.}
    \label{fig:stocks-lastmax}
\end{figure}

\begin{figure}[p!]
    \centering
    \includegraphics[width=0.7\textwidth]{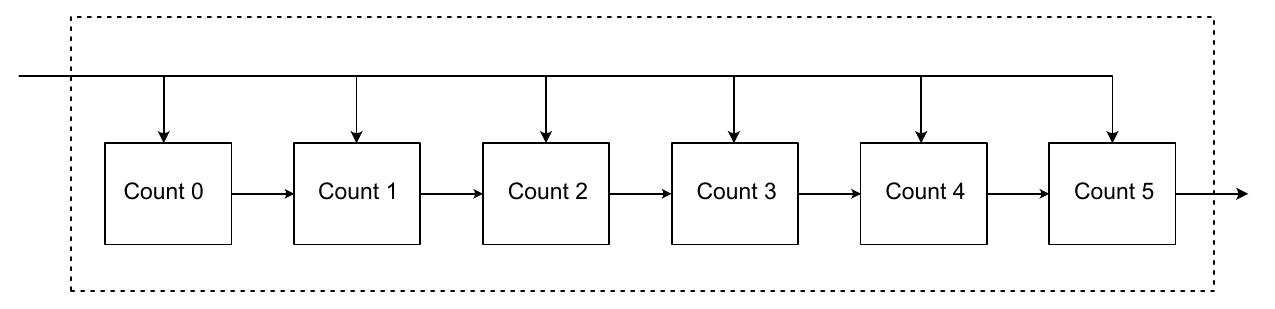}
    \caption{Diagram of Count.}
    \label{fig:stocks-count}
\end{figure}

Following the algorithm outlined above,
we construct a cascade to detect TTOP. Its diagram is shown in
Figure~\ref{fig:stocks-overall-diagram}.
The five boxes in the diagram correspond to five functionalities involved in
detecting TTOP, as described by Algorithm~\ref{algorithm:ttop}.
Each box is itself a cascade of one or more flip-flop semiautomata.
The leftmost arrow represents the current input $\mathit{cur}$.
The other arrows specify dependencies among components.
For instance, the state $\mathit{prev}$ of the Previous component is read by the
Slope, LastMax, and LastMin components.
The rightmost arrow is the output, which is a function of the state of the Count
component.
Next we describe the components one by one.
We will consider stock prices to be numbers that can be represented with $N$
bits.
In the following, let $\mathbb{B} = \{ 0,1 \}$. Then the current input
$\mathit{cur}$ is an element of $\mathbb{B}^N$.

\paragraph{Previous.}
The Previous component reads only the external input $\mathit{cur}$.
Intuitively, it has to store $\mathit{cur}$, so that it is available at the next
iteration.
It consists of $N$ flip-flop semiautomata. Its diagram is shown in
Figure~\ref{fig:stocks-previous}.
The flip-flop semiautomata are independent of each other. The $i$-th
flip-flop stores the $i$-th bit from the input.
Specifically, the $i$-th flip-flop semiautomaton (the Bit~$i$ component in the
diagram) is the flip-flop semiautomaton with input function $\phi_i$ defined as
follows.
Note that the input to the Previous component is 
$\mathit{cur} = \langle b_1, \dots, b_N \rangle \in \mathbb{B}^N$,
and component Bit~$i$ additionally receives the states 
$q_1, \dots, q_{i-1}$ of the preceding Bit components.
Component Bit~$i$ is a flip-flop semiautomaton with input function defined as
follows:
\begin{itemize}
  \item
    if $b_i = 1$, then return $\mathit{set}$,
  \item
    if $b_i = 0$, then return $\mathit{reset}$.
\end{itemize}
The initial state of every Bit component is $\mathit{low}$, so that the Previous
component initially stores stock price $0$.
This ensures that the first input is not deemed a local maximum.

\paragraph{Slope.}
The Slope component receives $\mathit{cur}$ and $\mathit{prev}$.
Its diagram is shown in Figure~\ref{fig:stocks-slope}.
Note that $\mathit{cur} \in \mathbb{B}^N$ and 
$\mathit{prev} \in \{ \mathit{low}, \mathit{high} \}^N$---and we can still see
it as an element of $\mathbb{B}^N$, seeing
$\mathit{low}$ as $0$ and $\mathit{high}$ as $1$.
The Slope component consists of one flip-flop semiautomaton, with input function
defined as follows:
\begin{itemize}
\item
if $\mathit{prev} > \mathit{cur}$, then return $\mathit{set}$,
\item
if $\mathit{prev} < \mathit{cur}$, then return $\mathit{reset}$ ,
\item
if $\mathit{prev} = \mathit{cur}$, then return $\mathit{read}$.
\end{itemize}
Its initial value is $\mathit{high}$, matching the initialisation specified in
the pseudocode.

\paragraph{LastMax.}
It receives $\mathit{cur}$, $\mathit{prev}$, $\mathit{slope}$.
It checks whether $\mathit{prev}$ is a new local maximum, and it stores it if it
is smaller than the currently stored maximum.
A flag is used to handle the special case of when no maximum has been stored
yet---thus implementing the logic of the algorithm, but in a slightly different
way; with a flag, instead of boundary values.

Note that $\mathit{cur} \in \mathbb{B}^N$,
$\mathit{prev} \in \{ \mathit{low}, \mathit{high} \}^N$,
$\mathit{slope} \in \{ \mathit{low}, \mathit{high} \}$---and for the slope, we
can think of $\mathit{low}$ as $\mathit{negative}$ and of $\mathit{high}$ as
$\mathit{positive}$.
Its diagram is shown in Figure~\ref{fig:stocks-lastmax}. Its components are
flip-flop semiautomata: Flag, Bit 1, ..., Bit N.
The input function of Flag is defined as follows:
\begin{itemize}
  \item
    if $\mathit{slope} = \mathit{positive}$ and $\mathit{prev} > \mathit{cur}$,
    then return $\mathit{set}$,
  \item
    else return $\mathit{read}$.
\end{itemize}
Note that Bit $i$ additionally receives $\mathit{flag} \in \{ \mathit{low},
\mathit{high} \}$ and the states $q_1, \dots, q_{i-1}$ of the previous bits.
The input function of Bit $i$ is defined as follows:
\begin{itemize}
  \item
    if $\mathit{slope} = \mathit{positive}$ and $\mathit{prev} > \mathit{cur}$,
    then

    {\footnotesize (This is the condition for checking when $\mathit{prev}$ is a
    local maximum.)}
    \begin{itemize}
      \item
        if $\mathit{flag} = \mathit{low}$, then
        \begin{itemize}
          \item
          if $\mathit{prev}_i = \mathit{high}$, then return $\mathit{set}$,
          \item
          if $\mathit{prev}_i = \mathit{low}$, then return $\mathit{reset}$,
        \end{itemize}

        {\footnotesize (It is the first time we encounter a $\mathit{prev}$ that is a valid
        local maximum. At this point, the number stored in the LastMax component
      is an arbitrary initialisation value. We overwrite it.)}
      \item
        otherwise, when $\mathit{flag} = \mathit{high}$,
        \begin{itemize}
          \item
              if $q_{1:i-1} = \mathit{prev}_{1:i-1}$, then

            {\footnotesize ($\mathit{prev}$ and the stored number are equal in
            their first $i-1$ most significant bits)}
            \begin{itemize}
              \item
                if $\mathit{prev}_i = 1$ then return $\mathit{read}$, 

            {\footnotesize (we keep the curently stored $i$-th bit,
            which amounts to taking the minimum since it is smaller than
            or equal to the $i$-th bit of $\mathit{prev}$)}
              \item
                if $\mathit{prev}_i = 0$ then return $\mathit{reset}$, 

            {\footnotesize (we write the $i$-th bit $0$
            of $\mathit{prev}$, which amounts to taking the minimum since it
            is smaller than or equal to the currently stored $i$-th bit)}
            \end{itemize}
          \item
            if 
            $q_{1:i-1} > \mathit{prev}_{1:i-1}$, then

            {\footnotesize ($\mathit{prev}$ is smaller than the stored number,
              so we overwrite the stored number)}
            \begin{itemize}
              \item
                if $\mathit{prev}_i = 1$, then return $\mathit{set}$, 
              \item
                if $\mathit{prev}_i = 0$, then return $\mathit{reset}$, 
            \end{itemize}
          \item
            if 
            $q_{1:i-1} < \mathit{prev}_{1:i-1}$ 
            then reutrn
            $\mathit{read}$,

            {\footnotesize (the stored number is smaller than $\mathit{prev}$,
            so we keep it)}
        \end{itemize}
    \end{itemize}
  \item
    otherwise return $\mathit{read}$.

    {\footnotesize ($\mathit{prev}$ is not a maximum, and hence we keep
      the currently stored maximum.)}
\end{itemize}
The Flag's state is initially set to $\mathit{low}$.
The initial state of the Bits does not matter, as it will be overwritten first.

\paragraph{LastMin.}
It is dual to LastMax.

\paragraph{Count.}
It receives $\mathit{cur}$, $\mathit{prev}$, $\mathit{slope}$,
$\mathit{lastMax}$, $\mathit{lastMin}$.
It checks whether a new extremum has been found, and it adjusts the count
accordingly. The count could increase or decrease, as described in
Algorithm~\ref{algorithm:ttop}.
Its diagram is shown in Figure~\ref{fig:stocks-count}.
It consits of six flip-flop semiautomata: Count 0, ..., Count 5.

Count $i$ receives $\mathit{cur}$, $\mathit{prev}$, $\mathit{lastMax}$,
$\mathit{lastMin}$ as well as the states $q_1, \dots, q_{i-1}$ of the previous
flip-flop semiautomata.  The input function of Count $i$ is defined as follows:
\begin{itemize}
  \item
    if $\mathit{slope} = \mathit{positive}$ and $\mathit{prev} > \mathit{cur}$
    then

    {\footnotesize ($\mathit{prev}$ is a new local maximum)}
    \begin{itemize}
      \item
        if $\mathit{prev} < \mathit{lastMax}$ then

        {\footnotesize ($\mathit{prev}$ can be part of the current sequence of
        extrema. We are going to increment the count, without exceeding $5$.)}

        \begin{itemize}
          \item
            if $q_{i-1} = \mathit{high}$ then return $\mathit{set}$
          \item
            if $q_{i-1} = \mathit{low}$ and $i = 5$, return $\mathit{read}$,
          \item
            if $q_{i-1} = \mathit{low}$ and $i \neq 5$, return $\mathit{reset}$,
        \end{itemize}
        \item
          otherwise, when $\mathit{prev} \geq \mathit{lastMax}$,

        {\footnotesize ($\mathit{prev}$ cannot be part of the current sequence of
        extrema. It is going to be the initial maximum, so we set the count to
      1.)}
      \begin{itemize}
        \item
          if $i = 1$, then return $\mathit{set}$,
        \item
          if $i \neq 1$, then return $\mathit{reset}$,
      \end{itemize}
    \end{itemize}
  \item
    if $\mathit{slope} = \mathit{negative}$ and $\mathit{prev} < \mathit{cur}$
    then

    {\footnotesize ($\mathit{prev}$ is a new local minimum)}
    \begin{itemize}
      \item
        if $\mathit{prev} > \mathit{lastMin}$ then

        {\footnotesize ($\mathit{prev}$ can be part of the current sequence of
        extrema. We are going to increment the count, without exceeding $5$.)}

        \begin{itemize}
          \item
            if $q_{i-1} = \mathit{high}$ then return $\mathit{set}$
          \item
            if $q_{i-1} = \mathit{low}$ and $i = 5$, return $\mathit{read}$,
          \item
            if $q_{i-1} = \mathit{low}$ and $i \neq 5$, return $\mathit{reset}$,
        \end{itemize}
        \item
          otherwise, when $\mathit{prev} \leq \mathit{lastMin}$,

        {\footnotesize ($\mathit{prev}$ cannot be part of the current sequence of
        extrema. It is going to be first minimum, so we set the count to
      2.)}
      \begin{itemize}
        \item
          if $i = 2$, then return $\mathit{set}$,
        \item
          if $i \neq 2$, then return $\mathit{reset}$.
      \end{itemize}
    \end{itemize}
  \item
    otherwise, return $\mathit{read}$.

    {\footnotesize ($\mathit{prev}$ is not a local extremum, we keep the
    current count.)}
\end{itemize}

The initial state encodes a count of zero.
Specifically, the initial state of Count 0 is $\mathit{high}$, and the others
have initial state $\mathit{low}$.

\paragraph{Output function of the automaton.}
To detect whether the current sequence matches the TTOP pattern,
it suffices that the output function returns $1$ when the state of Count 5 is
$\mathit{high}$ and $0$ otherwise.

\subsection{Example from Reinforcement Learning}

\begin{example}
  \label{example:reinforcement_learning}
  Reinforcement Learning agents learn functions of the history of past 
  observations that return the probability of the next observation and reward, 
  in order to act effectively. 
  In the Cookie Domain from \citep{toroicarte2019learning},
  an agent has to learn to collect cookies which appear upon
pushing a button. 
In particular, the \emph{cookie domain}, shown in Figure~\ref{fig:d_cookie}, has three rooms
connected by a hallway. The agent (purple triangle) can move in the four
cardinal directions. There is a button in the orange room that, when pressed,
causes a cookie to randomly appear in the green or blue room. The agent receives
a reward of $+1$ for reaching (and  thus eating) the cookie and may then go and
press the button again. Pressing the button before reaching a cookie will remove
the existing cookie and cause a new cookie to randomly appear. There is no
cookie at the beginning of the episode. This domain is partially observable
since the agent can only see what it is in the room that it currently occupies,
as shown in Figure~\ref{fig:d_view}.
Specifically, at each step the agent receives a set of Boolean properties
holding at that step. The available properties are:
\begin{align*}
\mathcal{P} = \{\mathit{cookie}, \mathit{cookieEaten}, \mathit{buttonPushed},
\mathit{greenRoom},\mathit{orangeRoom},\mathit{blueRoom}, \mathit{hallway} \}
\end{align*}
These properties are true in the following situations:
$\mathit{greenRoom}$, $\mathit{orangeRoom}$, $\mathit{blueRoom}$ are true if the
agent is in a room of that color; 
$\mathit{hallway}$ is true if the agent is in the hallway;
$\mathit{cookie}$ is true if the agent is in the same room as a cookie;
$\mathit{buttonPushed}$ is true if the agent pushed the button with its last
action; and 
$\mathit{cookieEaten}$ is true if the agent ate a cookie with its last action.

In this
domain, the function predicting the probability of the next observation and reward is a group-free
regular function. The optimal policy that the agent has to learn is a group-free regular function as
well.
\end{example}

\begin{figure}[t]
    \centering
    \begin{subfigure}[b]{.49\textwidth}
        \centering
        \includegraphics[width=0.5\textwidth]{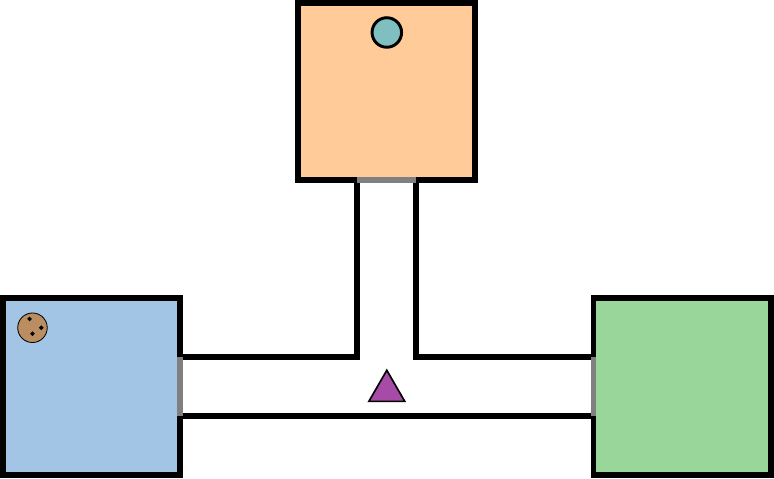}
        \subcaption{Cookie domain.}
        \label{fig:d_cookie}
    \end{subfigure}
    \begin{subfigure}[b]{.49\textwidth}
        \centering
        \includegraphics[width=0.5\textwidth]{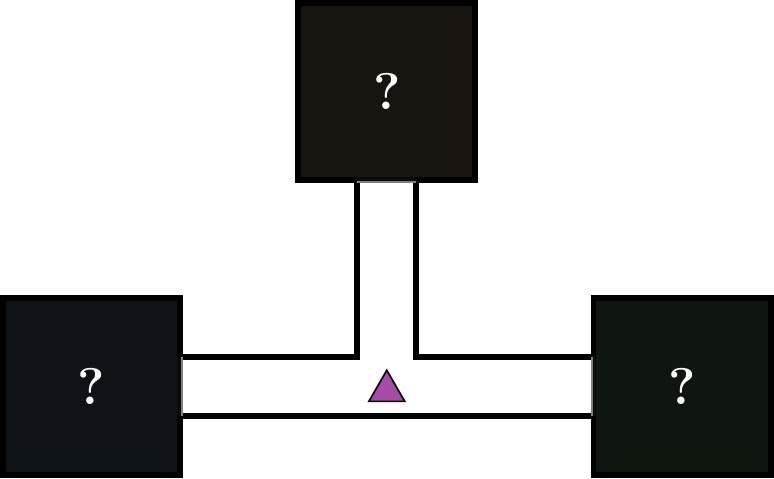}
        \subcaption{Agent's view.}
        \label{fig:d_view}
    \end{subfigure}
    \caption{In the cookie domain, the agent
      can only see what is in the current room.\\ Figures from [Toro Icarte et
        al., 2019].}
    \label{fig:domains}
\end{figure}

Here we show that the dynamics of the Cookie domain are described by a
star-free regular function. In particular, there is a cascade of flip-flops,
whose state along with the current input is sufficient to predict the evolution
of the domain.

\begin{figure}[t]
    \centering
    \includegraphics[width=0.7\textwidth]{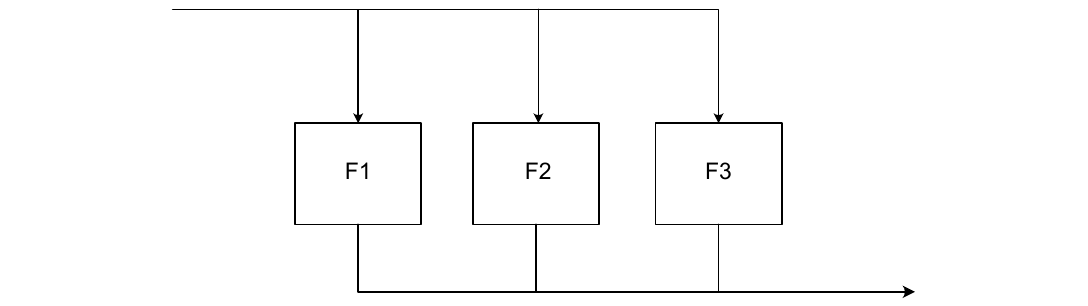}
    \caption{Diagram of the cascade for the Cookie Domain.}
    \label{fig:cookie-cascade}
\end{figure}

We show next that the dynamics of the Cookie domain are group-free regular as
they can be implemented by a cascade of three flip-flops $F_1,F_2,F_3$.
In fact, by three flip-flops that are independent of each other.
The diagram of the cascade is shown in Figure~\ref{fig:cookie-cascade}.

\paragraph{F1: Cookie around.}
The input function of the flip-flop semiautomaton $F_1$ is defined as follows:
\begin{itemize}
  \item
  if $\mathit{buttonPushed}$ is true, then return $\mathit{set}$,
  \item
  if $\mathit{cookieEaten}$ is true, then return $\mathit{reset}$,
  \item
  otherwise return $\mathit{read}$.
\end{itemize}
The initial state is $\mathit{low}$.
It is easy to see that the state of the semiautomaton is $\mathit{high}$ iff
there is a cookie in the blue or green room.

\paragraph{F2: Rooms visited.}
The input function of the flip-flop semiautomaton $F_2$ is defined as follows:
\begin{itemize}
  \item
    if $\mathit{buttonPushed}$ is true, then return $\mathit{reset}$,
  \item
    if $\mathit{greenRoom}$ is true, then return $\mathit{set}$,
  \item
    if $\mathit{blueRoom}$ is true, then return $\mathit{set}$,
  \item
    otherwise return $\mathit{read}$,
\end{itemize}
The initial state is $\mathit{low}$.
It is easy to see that the state of the semiautomaton is $\mathit{high}$ iff
the blue room or the green room have been visited since pushing the button.

\paragraph{F3: Room with cookie.}
The input function of the flip-flop semiautomaton $F_3$ is as follows:
\begin{itemize}
  \item
    if $\mathit{greenRoom}$ is true, and 
    $\mathit{cookie}$ is true, then return $\mathit{set}$,
  \item
    if $\mathit{greenRoom}$ is true, and 
    $\mathit{cookie}$ is false, then return $\mathit{reset}$,
  \item
    if $\mathit{blueRoom}$ is true, and 
    $\mathit{cookie}$ is true, then return $\mathit{reset}$,
  \item
    if $\mathit{blueRoom}$ is true, and 
    $\mathit{cookie}$ is false, then return $\mathit{set}$,
  \item
    otherwise $\mathit{read}$.
\end{itemize}
It is easy to see that, when there is a cookie in the blue or green room, and the
blue or green room have been visited since pushing the button, the state of the
semiautomaton is $\mathit{high}$ if the cookie is in the green room, and it is 
$\mathit{low}$ if the cookie is in the blue room.
The initial state does not matter, since the state of this semiautomaton will
not be used by the output function unless the blue room or the green room have
been visited.

\paragraph{Output function.}
The information provided by the states of the three flip-flops $F_1,F_2,F_3$
together with the current input is sufficient to predict the evolution of the
Cookie Domain.

Here we focus on the central aspect of the prediction, i.e., predicting whether
the agent will find a cookie upon entering a room. We omit predicting other
aspects, which regard location and can be predicted from the current location,
with no need of resorting to automaton's states.

The presence of a cookie is predicted as follows:
\begin{itemize}
  \item
    If the agent enters the hallway or the orange room, then the agent will not
    find a cookie.

  \item
    If the agent enters any room, and the state of $F_1$ is $\mathit{low}$, then
    the agent will not find a cookie.

  \item
    If the agent enters the green room, 
    the state of $F_1$ is $\mathit{high}$, 
    the state of $F_2$ is $\mathit{high}$, and 
    the state of $F_3$ is $\mathit{high}$, then 
    the agent will find a cookie.

  \item
    If the agent enters the blue room, 
    the state of $F_1$ is $\mathit{high}$, 
    the state of $F_2$ is $\mathit{high}$, and 
    the state of $F_3$ is $\mathit{high}$, then 
    the agent will not find a cookie.

  \item
    If the agent enters the green room, 
    the state of $F_1$ is $\mathit{high}$, 
    the state of $F_2$ is $\mathit{high}$, and 
    the state of $F_3$ is $\mathit{low}$, then 
    the agent will not find a cookie.

  \item
    If the agent enters the blue room, 
    the state of $F_1$ is $\mathit{high}$, 
    the state of $F_2$ is $\mathit{high}$, and 
    the state of $F_3$ is $\mathit{low}$, then 
    the agent will find a cookie.

  \item
    If the agent enters the green room, 
    the state of $F_1$ is $\mathit{high}$, and 
    the state of $F_2$ is $\mathit{low}$, then 
    the agent will find a cookie with probability $0.5$.

  \item
    If the agent enters the blue room, 
    the state of $F_1$ is $\mathit{high}$, and 
    the state of $F_2$ is $\mathit{low}$, then 
    the agent will find a cookie with probability $0.5$.
\end{itemize}